%% file: newcv.tex
\documentclass{article}

\usepackage[final]{template} % Uncomment for the camera-ready ``final'' version
\usepackage{graphicx}
\usepackage{subcaption} 
\graphicspath{ {figures/} }
\usepackage{math}
\def\Gsv{{\wt G^{\mathrm{S}, \wh v}}}
\def\Gsaq{{\wt G^{\mathrm{SA}, \wh q}}}
\def\Gtraj{{\wt G^{\mathrm{Traj}, \wh q}}}
\input{shortcuts}
\usepackage{xspace}
\newcommand\blfootnote[1]{%
	\begingroup
	\renewcommand\thefootnote{}\footnote{#1}%
	\addtocounter{footnote}{-1}%
	\endgroup
}

\title{Trajectory-wise Control Variates for Variance Reduction in Policy Gradient Methods}

\author{
  Ching-An Cheng\textsuperscript{*}\\
  Georgia Tech \\
  \texttt{cacheng@gatech.edu}
  \And 
  Xinyan Yan\textsuperscript{*} \\
  Georgia Tech \\
  \texttt{xinyan.yan@cc.gatech.edu}
  \And 
  Byron Boots 
  \\
  Georgia Tech\\
  \texttt{bboots@cc.gatech.edu}
}

\begin{document}
\maketitle

%===============================================================================
\begin{abstract}
	Policy gradient methods have demonstrated success in reinforcement learning tasks that have high-dimensional continuous state and action spaces. However, policy gradient methods are also notoriously sample inefficient. This can be attributed, at least in part, to the high variance in estimating the gradient of the task objective with Monte Carlo methods. 
	Previous research has endeavored to contend with this problem by studying control variates (CVs) that can reduce the variance of estimates without introducing bias, including the early use of baselines, state dependent CVs, and the more recent state-action dependent CVs.
	In this work, we analyze the properties and drawbacks of previous CV techniques and, surprisingly, we find that these works have overlooked an important fact that Monte Carlo gradient estimates are generated by \emph{trajectories} of states and actions. We show that ignoring the correlation across the trajectories can result in suboptimal variance reduction, and we propose a simple fix: a class of \emph{trajectory-wise} CVs, that can further drive down the variance.
	We show that constructing trajectory-wise CVs can be done recursively and requires only learning state-action value functions like the previous CVs for policy gradient.
	We further prove that the proposed trajectory-wise CVs are optimal for variance reduction 
	under reasonable assumptions.
	%We analyze the role that modeling error and value function estimation error play in variance reduction, which guide us to design practical CVs for 	various forms of policy gradient.
	%Finally, we introduce practical algorithms and demonstrate their effectiveness on several challenging continuous control tasks in OpenAI Gym.
\end{abstract}
  {\blfootnote{\textsuperscript{*}
		Equal contribution.}}
\keywords{Reinforcement Learning, Policy Gradient, Control Variate} 

%===============================================================================

%\vspace{-1mm}
\section{Introduction}

%\paragraph{Need for variance reduction} 
Policy gradient methods~\cite{williams1992simple,sutton2000policy,kakade2002natural,peters2008natural,schulman2015trust,cheng2018predictor} are a popular class of model-free reinforcement learning (RL) algorithms. They have many advantages, including straightforward update rules and well-established convergence guarantees~\cite{sutton2000policy,konda2000actor,cheng2018fast,yang2019policy}. 
However, basic policy gradient methods, like REINFORCE~\cite{williams1992simple}, are also notorious for their sample inefficiency. 
% which restricts their use to mainly simulated environments.
This can be attributed, at least in part, to the high variance in Monte Carlo gradient estimates, which stems from both policy stochasticity necessary for exploration as well as stochastic environmental dynamics. The high variance is further exacerbated as the RL horizon becomes longer and higher dimensional.
If the variance of gradient estimates can be reduced, then the learning speed of policy gradient methods can be accelerated \cite{ghadimi2016mini,cheng2018predictor}.

%can accelerate the stochastic optimization methods. 
%CVs that largely reduce the variance can accelerate the stochastic optimization method
The reduction of variance in policy gradients thus is an important research topic in RL, which has been studied since early work on the development of policy gradient methods.
For example, function approximators (critics) have been adopted to (partially) replace the Monte Carlo estimates of accumulated costs to  reduce variance but at the expense of introducing bias in the search direction ~\cite{sutton2000policy,kimura2000analysis,thomas2014bias,silver2014deterministic,schulman2015high,sun2018truncated}. 
This bias-variance tradeoff can work well in practice, but can also potentially cause divergent behaviors and requires careful tuning
%especially in the later stages of learning 
~\cite{schulman2015high,efroni2018beyond,yang2019policy}.

Another line of research uses the control variate (CV) method from statistics, designed for reducing variance in Monte Carlo methods without introducing bias~\cite{sutton2000policy,ng1999policy,greensmith2004variance,jie2010connection,gu2016q,liu2017action,grathwohl2017backpropagation,tucker2018mirage,pankov2018reward}. For policy gradient algorithms, specialized CV methods have been proposed in order to take advantage of structure inherent in RL problems. 
The CV method works by defining a certain correlation function (called the control variate) that approximates the Monte Carlo samples and yet yields a closed-form (or low-variance approximation of) the expectation of interest.
%
%General CVs 
%have demonstrated effectiveness 
%in variance reduction in stochastic optimization problems.
%without assuming much about the problem structure, \eg,, taking average of gradients from recent iterates~\cite{}.
%For RL problems, the MDP structure should be taken into account to define specialized and more powerful CVs. 
For policy gradient especially, the state dependent CVs (also known as baselines or reward reshaping~\cite{ng1999policy,jie2010connection})  have been thoroughly investigated ~\cite{greensmith2004variance}.
%, minimum variance unbiased baselines }).
Common state dependent CVs are constructed as approximators of the policy's value function, which admits update rules based on standard policy evaluation techniques. 
Overall, state dependent CVs are simple to implement and have been found to be quite effective. % compared with REINFORCE.
However, the resulting policy gradients can still posses detrimentally  high variance, especially in  problems that has a long horizon.
%have a longer horizon and higher dimensions.
%
%compatible parameterization for Q-function \cite{sutton2000policy},
This  has motivated the recent development of state-action dependent CVs~\cite{gu2016q,liu2017action,grathwohl2017backpropagation,tucker2018mirage,wu2018variance}. 
By using more elaborate CVs, these techniques can further reduce the variance due to randomness in the \emph{actions} in gradient estimates that the previous state-only CVs fails to manage. 
%Empirically, state-action CVs have been observed to give better results as compared with state dependent CVs.\footnote{\citet{tucker2018mirage} noted some early state-action dependent CVs have flaws in implementation.} 

Considering the decades-long development of CV methods, one might wonder if there is a need for new policy gradient CV techniques. % for policy gradient, in addition to improving the quality of correlation functions.
In this paper, we argue that the past development of CVs for policy gradients has overlooked an important fact that the Monte Carlo gradient estimates are generated by rolling out a policy and collecting statistics along a \emph{trajectory} of states and actions. Instead the focus has been on sampling \emph{pairs} of states and actions, ignoring the correlation between states and actions \emph{across} time steps.
Recently \citet{tucker2018mirage} empirically analyzed the variance of instantaneous state-action pairs and compared this to the variance correlations across time steps in multiple simulated robot locomotion tasks. They found that the variance due to long-term trajectories is often larger than the variance due to instantaneous state-action pairs. % that has been the main focus of previous research. 
This finding implies that there is potential room for improvement. 

In this paper, we theoretically analyze the properties of previous CVs, and show that indeed the variance due to long-term trajectories has non-negligible effects. Motivated by this observation, we propose a family of trajectory-wise CVs, called TrajCV, which augment existing CVs with extra terms to \emph{additionally} cancel this long-term variance. We show that TrajCV is particularly effective when the transition dynamics, despite unknown, is close to deterministic.
Like existing CVs, TrajCV requires only knowledge of state-action value function (\ie Q-function) approximates, and can be computed recursively. 
Furthermore, we prove that TrajCV is optimal for variance reduction under reasonable assumptions.
These theoretical insights are validated in simulation.

%We note that 
Upon finishing this work we discovered a recent technical report~\cite{pankov2018reward} that is motivated similarly and details exactly the equation~\eqref{eq:traj-wise difference estimator} that TrajCV uses. Their empirical results on simulated LQG tasks are encouraging too: TrajCV demonstrated superior performance compared with previous state and state-action dependent CVs.
By contrast, we derive TrajCV following a completely different route, which brings extra insight into the previous deficiency and suggests natural ways for improvement. In addition, we analyze other potential trajectory-wise CVs and prove the proposed idea is optimal.

%\vspace{-2mm}
\section{Problem Setup and Background} \label{sec:background}
%\vspace{-2mm}

We consider episodic policy optimization in a finite-horizon Markov Decision Process (MDP)~\cite{bellman1957markovian,bertsekas1995dynamic} with horizon $h$, state space $\Sc$, action space $\Ac$, instantaneous cost function $c: \Sc \times \Ac \to \R$, 
% dynamics $f: \Sc \times \Ac \to \Sc$ (which can be probabilitic, \ie , $f_{x,u}$ defines a ),
initial state distribution $p_1$, and 
dynamics $d$.\footnote{We use one-based indexing throughout the manuscript.} %\footnote{For all $s \in \Sc$ and $a \in \Ac$, $d_{s,a}$ is a density with respect to a measure $\mu_{\Sc}$ on $\Sc$. For example, measure $\mu_\Sc$ can be the Lebesgue measure if the state space is real-valued, or the counting measure if the state space if finite.}),
%\footnote{With respect to the measure $\mu_{\Sc}$.}
Given a parameterized \emph{stochastic} policy class $\Pi$
%\footnote{Fix $\pi \in \Pi$ and $s \in \Sc$, $\pi_s$ defines a density with respect to a measure $\mu_{\Ac}$ on $\Ac$.}, 
the goal is to search for a policy in $\Pi$
that achieves low accumulated costs averaged {over trajectories}
\begin{align} \label{eq:objective}
J(\pi) := \Exp   \br{ C_{1:h}},  \quad \text{where  } C_t := c(S_t, A_t), \quad S_1 \sim p_1, \quad A_t \sim \pi_{S_t}, \quad S_{t+1} \sim d_{S_t, A_t}
\end{align}
where
$d_{s,a}$ denotes the distribution of the next state after applying action $a\in\Ac$ at state $s\in\Sc$, and $\pi_s$ denotes  the distribution of action at state $s \in \Sc$.
Note that $S_t$ and $A_t$ are the sampled state and action at step $t$, and $_:$ denotes summation (\ie  $C_{1:h} = \sum_{t=1}^{h} C_t$).
For simplicity of writing, we embed the time information into the definition of state, \eg, $c(S_t, A_t)$ can represent non-stationary functions.
%and $\pi$ denotes a policy with certain parameter. 
%We assume that $\pi \in \Pi$ can be reparameterized as 
%$\pi: \Sc \times \Ec \to \Ac$, where $\Ec$ is a probability space. 
%Because the transition is deterministic and state start is fixed\footnote{In order to keep the notations clean, the account for transition probability and start state distribution is deferred until Appendix.},
The randomness in \eqref{eq:objective} consists of 
%The expectation in~\eqref{eq:objective} is over 
the randomness in the start state, policy, and dynamics.
In this work, we focus on the case where the dynamics  $d$ and the start state distribution $p_1$ are unknown, but the instantaneous cost $c$ is known. 
%We also assume that the problem can be reset according to $p_1$, so that policies can be optimized through trajectory episodes.

For notation, we will use uppercases to denote random variables, such as $S_t$ and $A_t$, with the exception of $J$. We will be frequently manipulating conditional distributions.
 We adopt the subscript notation below to write conditional expectation and variance. For $\Exp_{X|Y}[f(X,Y)]$ of some function $f$, $X$ denotes the random variable over which the expectation is defined and $Y$ denotes the conditioned random variable. Furthermore, for $f(X_1, \dots, X_N, Y)$, we use $\Exp_{|Y}[f(X_1, \dots, X_N, Y)]$ as a shorthand to denote taking the expectation over all other random variables (\ie  $X_1, \dots, X_N$) conditioned on $Y$. This subscript notation also applies to variance, which is denoted as $\Var$.

\subsection{Policy Gradient Methods: Pros and Cons}
%\vspace{-1mm}

The goal of this paper is to improve the learning performance of policy gradient methods~\cite{williams1992simple,konda2000actor,sutton2000policy,kakade2002natural,peters2008natural,silver2014deterministic,schulman2015trust,cheng2018predictor}. These algorithms treat minimizing \eqref{eq:objective} as a first-order stochastic non-convex optimization problem, where noisy, unbiased gradient estimates of $J$ in \eqref{eq:objective} are used to inform policy search. The basic idea is to apply the likelihood-ratio method to derive the gradient of \eqref{eq:objective}. Let us define $ N_t  := \nabla \log \pi_{S_t}(A_t)$, where $\nabla$ is the derivative with respect to the policy parameters, and define $q^{\pi}$ as the Q-function of $\pi$; that is, $q^{\pi}(S_t, A_t) = \Exp[C_{t:h}]$ where the expectation is generated by taking $A_t$ at $S_t$ and then $\pi$ afterwards.
Define $G :=G_{1:h}$ and $G_t :=  N_t C_{t:h}$.
Then it follows%the gradient of $J$ can  be expressed as
~\cite{williams1992simple}
\begin{align}  \label{eq:policy gradient}
\textstyle \nabla J(\pi) =\Exp[\sum_{t=1}^{h} N_t q^{\pi}(S_t, A_t)]  = \Exp\br{G},
%G_{0:h}, \quad G_t :=  N_t C_{t:h}, \quad N_t  := \nabla \log \pi_{S_t}(A_t)\textit{}
\end{align}
where the second equality is due to  $q^{\pi}(S_t, A_t) = \Exp[C_{t:h}]$.
Equation \eqref{eq:policy gradient} is an expectation over trajectories generated by running $\pi$. 
Therefore we can treat the random vector $G$ as an unbiased estimate of $\nabla J(\pi)$, which can be computed by executing the policy $\pi$ starting from distribution $p_1$ and then recording the statistics $G_t$, for $t\in \{1,\dots, h\}$. % to define $G$.
This technique is known as the Monte Carlo estimate of the policy gradient, which samples i.i.d. trajectories from the trajectory distribution defined in \eqref{eq:objective} to approximate the expectation. 
%Importantly,  $G$ is a function of state-action trajectory $S_{1:h}, A_{1:h}$ 

The policy gradient methods (\eg REINFORCE~\cite{williams1992simple}) optimize policies based on gradient estimates constructed using the above idea. % in \eqref{eq:policy gradient}.
They have numerous advantages. 
For example, they have straightforward update rules %and stable and steady iterative improvement, 
and convergence guarantee, as they minimizes $J$ directly by updating parameters in descent directions in expectation~\cite{sutton2000policy,konda2000actor,cheng2018fast,yang2019policy}.
However, simply using the Monte Carlo
estimate $G$ in policy optimization (\ie the vanilla implementation of REINFORCE) can result in poor parameter updates due to excessive variance~\cite{silver2014deterministic,schulman2015high}.
Therefore, while ideally one can apply standard first-order optimization algorithms such as mirror descent~\cite{beck2003mirror} using $G$ to optimize the policy, this often is not viable in practice. They would in turn require a tremendous amount of trajectory rollouts in a single update step in order to attenuate the high variance, making learning sample inefficient.

The high variance of $G$ is due to the exploration difficulty in RL: in the worst-case, the variance of $G$ can grow exponentially in the problem's horizon $h$~\cite{kakade2003sample,vemula2019contrasting}, as it becomes harder for the policy to visit meaningful states and get useful update information. Intuitively we can then imagine that policy optimization progress can be extremely slow, when the gradient estimates are noisy. %, if we directly treat $G$ as the search direction.
From an optimization perspective, variance is detrimental to the convergence rate in stochastic optimization. 
%In the settings of first-order stochastic optimization methods, the variance of gradient estimate plays an important role in convergence rate. 
For example, the number of iterations for mirror descent to converge to an $\epsilon$-approximate stationary point is {$O ((\Tr \paren{\Var[G]}+1)/\epsilon^2)$}, increasing as the problem becomes more noisy~\cite{ghadimi2016mini}.
Therefore, if the variance of estimates of \eqref{eq:policy gradient} can be reduced, the policy gradient methods can be accelerated.
%can accelerate the stochastic optimization methods. 
%CVs that largely reduce the variance can accelerate the stochastic optimization method

%\vspace{-1mm}
\subsection{Variance Reduction and Control Variate}

A powerful technique for reducing the variance in the Monte Carlo estimates is the CV method~\cite{ross1990course,mcbook}. 
Leveraging correlation between random estimates, the CV method has formed the backbone of many state-of-the-art stochastic optimization algorithms %(including SAG, SAGA, and SVRG
~\citep{schmidt2017minimizing,johnson2013accelerating,defazio2014saga}.
%, and its effectiveness has been validated in a range of applications~\red{\cite{??}}. 
The use of CV methods has also proved to be critical to designing practical policy gradient methods for RL~\cite{greensmith2004variance}, because of the high-variance issue of $G$ discussed in the previous section.
%that reduce the variance of the components of $G$, {\ie  } the variance of $G_t$.
%
Below, we review the basics of the CV method as well as previous techniques designed for reducing the variance of $G$.
Without loss of generality, we suppose only one trajectory is sampled from the MDP  to construct the estimate of \eqref{eq:policy gradient} and study the variance of different single-sample estimates. 
We remind that the variance can be always further reduced, when more i.i.d. trajectories are sampled.

% in the number of IID trajectories. 

%\paragraph{Assumptions on nearly deterministic dynamics}
%Why this is a realistic assumption. 

%\vspace{-1mm}
\subsection{The Control Variate Method} \label{sec:simple CV}

Consider the problem of estimating the expectation $\Exp [X]$, where $X$ is a %continuous-valued 
(possibly multivariate) random variable. % distributed according to  distribution $P_X$.
%Sampling IID $X \sim P_X$, though being unbiased and very computationally efficient, can sometimes lead to drastic variance.
The CV method~\cite{ross1990course,mcbook} is a technique for synthesizing unbiased estimates of $\Exp [X]$ that potentially have lower variance than the naive sample estimate $X$. %\footnote{Recall that, without loss of generality, we focus on single-sample estimates.}
%is expected to control the variance of the sample estimate while remaining unbiased and computationally appealing.
It works as follows: 
assume that we have access to another random variable $Y$, 
called the \emph{CV}, 
%which we call the CV, 
whose expectation $\Exp [Y]$ is cheaper to compute than $\Exp [X]$. Then we can devise this new estimate by a linear combination: % of the two random variables:
\begin{align} \label{eq:estimate}
X - \alpha^\top (Y - \Exp [Y]),
\end{align}
where $\alpha$ is a properly-shaped matrix.
%\ie , neutralizing the original estimate  by $A^\intercal (Y - \Exp [Y])$.
Due to the linearity of expectation, the estimate in \eqref{eq:estimate} is unbiased, and its trace of variance (\ie the size of variance) can be lower-bounded as~\cite{wang2013variance}
\begin{align} \label{eq:cv introduction}
&\eqsp 
\Tr \paren{\Var[X - \alpha^\top(Y - \Exp [Y])]}
%\nonumber \\&= \Var[X - \alpha Y]
%\nonumber \\&= \Var[X] - 2 \alpha \Cov[X, Y] + \alpha^2 \Var[Y]
%\nonumber \\&
\ge 
\Tr \paren{\Var[X] - 
\Cov[X, Y] \Var[Y]^{-1} \Cov[Y, X]}
\end{align}
Suppose $Y$ is in the same dimension as $X$.
One can show that when $\alpha$ is optimally chosen as
$
\alpha^\star = \frac{1}{2}\Var[Y]^{-1} \paren{\Cov[X,Y] + \Cov[Y,X]}
$.
When data are too scarce to estimate $\alpha^\star$, a practical alternative is setting $\alpha$ as the identity matrix, which often works well when $Y$ is positively correlated with $X$. 
%\yan{Seem that variance reduction is not guaranteed though. Maybe delete the so long as sub-sentence}
The resulting estimate $X - (Y - \Exp \br{Y})$ is known as the \emph{difference estimator}~\cite{mcbook} and has variance 
$
\Var \br{X - Y}
$, meaning that if $Y$ is close to $X$ then the variance becomes smaller.
In the following, we concentrate on the design of difference estimators; we note that designing a good $\alpha$ is an orthogonal research direction.

%\subsection{Notation and Convention}
%%The symbol comma denotes the range of the subscripts of the items in a set, \eg,, $s_{t,h} := \{s_k\}_{k=t}^{h}$.
%%The symbol colon is used to denote summation, \eg,, $C_{t:h} := \sum_{k=t}^{h} C_k$.
%%We use one-based indexing.
%%We use letters in uppercase to denote random variables, and Greek letters in lowercase to denote constants. 
%For expectation and variance, we adopt subscript notation, \ie , for 
%$\Exp_{X|Y}$, $X$ denotes the random variable to take expectation over, and $Y$ denotes the random variable to condition on. 
%In particular, we use $\Exp_{|Y}$ to mean taking expectation over all other random variables conditioned on $Y$.	
%Let $\Var \br{X} = \Tr \paren{\Cov(X, Y)}$.
%\begin{enumerate}
%	\item $X - \paren{Y - \Exp \br{Y}}$: estimate. 
%	\item $Y - \Exp \br{Y}$: offset.
%	\item $Y$: CV.
%\end{enumerate}

% WITHOUT \alpha
%\begin{align*}
%X - (Y - \Exp [Y])
%\end{align*}
%which is apparently unbiased.
%When the CV $Y$ is correlated with $X$, the variance of the new estimate, $\Var[X-Y]$, is smaller then $\Var[X]$, leading to variance reduction. 

%\vspace{-2mm}
\subsection{Common Control Variates for Policy Gradient Methods} \label{sec:common cv}
%\vspace{-2mm}
%{State-Action-Dependent Control Variate}

The art to various CV methods lies in the design of the correlated random variable $Y$. The choice is often domain-dependent, based on how %the random variable 
$X$ is generated.
When estimating the policy gradient in~\eqref{eq:policy gradient}, many structures (\eg the Markov property) can be leveraged to design CVs, as we shall discuss. 
%
%For simplicity of exposition, we focus on the component policy gradient $G_t$ of $G$ given in \eqref{eq:policy gradient}, 
%as the CV for $G$ can be directly derived by summing up the CVs for each $G_t$ (as expectation is linear) or linearly combined
%as the CV for $G$ can be derived by linearly combining the CVs for each $G_t$\footnote{Due to space limitation, discussion on combining control variates is deferred to \cref{app:G_t and G}.}~\eqref{eq:estimate}.
%
Following previous works (\eg \cite{greensmith2004variance,tucker2018mirage}) here we focus on the policy gradient component $G_t$ of $G$ given in \eqref{eq:policy gradient} for simplicity of exposition.\footnote{
	Without any assumption of the MDP, the variance of $G$ can be bounded by the variance of $G_t$ (\cref{app:pg var bound}). Tighter bounds can be derived when assumptions on the MDP is made, \eg, faster mixing rate~\cite{greensmith2004variance}.
} 
The most commonly used CVs for policy gradient~\cite{williams1992simple, ng1999policy,greensmith2004variance} are %in the form of 
state-dependent functions $\widehat{v}:\SS\to\R$, which leads to the difference estimator 
\begin{align} \label{eq:state cv}
%\widetilde{G}_t 
\Gsv_t
:=  G_t - \paren{N_t \wh{V}_t  - \Exp_{A_t|S_t}[N_t \wh{V}_t]} =
% N_t(C_{t:h} - \wh{V}_t), 
G_t - N_t \wh V_t,
\quad \text{ where }\wh{V}_t := \wh{v}(S_t),
\end{align}
%where $\wh{V}_t := \wh{v}(S_t)$, 
and the expectation vanishes as $\Exp_{A_t|S_t}[N_t \wh{V}_t] =  \wh{V}_t \nabla \Exp_{A_t|S_t}[ 1] = 0$.\footnote{State dependent functions naturally include non-stationary constant baselines in our notation.}
Recently, \emph{state-action CVs}  $\wh{q}: \Sc \times \Ac \to \R$ have also been proposed~\cite{gu2016q, liu2017action,grathwohl2017backpropagation,tucker2018mirage,	wu2018variance,ciosek2018expected}, in an attempt to reduce more variance through CVs that better correlate with $G_t$.
The state-action CVs yields the difference estimator
%Define $\wh{Q}_t := \wh{q}(S_t, A_t)$; then the state-action CV leads to the difference estimator
\begin{align} \label{eq:state-action cv}
%\wt{G}_t 
\Gsaq_t
:=  G_t - \paren{N_t \wh{Q}_t  - \Exp_{A_t|S_t}[N_t \wh{Q}_t]}, 
\quad \text{ where } \wh{Q}_t := \wh{q}(S_t, A_t).
\end{align}
Usually $\wh{v}$ and $\wh{q}$ are constructed as function approximators of the value function $v^\pi$ and the Q-function $q^\pi$ of the current policy $\pi$, respectively, and learned by policy evaluation, \eg, variants of TD($\lambda$) \cite{singh1996reinforcement}, during policy optimization.
Therefore, these methods can also be viewed as unbiased actor-critic approaches. 
In practice, 
it has been observed 
that these CVs indeed %reduce the variance of policy gradient estimates,
 accelerate policy optimization, especially in challenging simulated robot control tasks ~\cite{gu2016q, liu2017action,tucker2018mirage,	wu2018variance,ciosek2018expected}.

%\vspace{-2mm}
\section{Why We Need New Control Variates}
\label{sec:why we need new cvs}
%\vspace{-2mm}
%In this paper, we argue that the past development of CVs for policy gradients has overlooked an important fact that Monte Carlo gradient estimates are generated by rolling out a policy and collecting statistics along a \emph{trajectory} of states and actions. Instead the focus has been on sampling \emph{pairs} of states and actions, ignoring the correlation between states and actions \emph{across} time steps.
%Recently \citet{tucker2018mirage} empirically analyzed the variance of instantaneous state-action pairs and compared this to the variance correlations across time steps in a range of simulated robot locomotion tasks. They found that the variance originating from correlations across time steps in trajectories is often larger than the variance in instantaneous state-action pairs that has been the main focus of previous research. This finding implies that there is potential room for improvement. % by reducing $\Var_{|S_t, A_t}$. 

Given the decades-long development of CVs for policy gradient reviewed above, one might wonder if there is a need for new policy gradient CV techniques. If so, what is the additional gain we can potentially have? 
To answer this question, let us first analyze the variance of policy gradient component  $G_t$ and how the CVs above reduce it.
By the law of total variance, $\Var[G_t]$ can be decomposed into \emph{three terms} 
%attributed to
%the randomness in 1) state $S_t$, 2) action $A_t$,  and 3) the trajectory after $t$, respectively,
\begin{align} \label{eq:G_t variance}
%& \eqsp \Var \br{G_t} 
%\nonumber\\&= \Var_{S_t} \br{ \Exp_{|S_t} \br{ N_t C_{t:h}}} 
%+ \Exp_{S_t} \br{ \Var_{|S_t} \br{ N_t C_{t:h}}}
%\nonumber\\&
%=  
\Var_{S_t}  \Exp_{|S_t} \br{ N_t C_{t:h}}
+ \Exp_{S_t}  \Var_{A_t|S_t} \br{  N_t \Exp_{|S_t, A_t} \br { C_{t:h}}}  + \Exp_{S_t, A_t} \Var_{|S_t, A_t} \br {N_t C_{t:h} },
\end{align}%
where the first term is due to the randomness of policy and dynamics before getting to  $S_t$, the second term is due to policy randomness alone at step $t$, \ie selecting $A_t$, and the third term is due to again both the policy and the dynamics randomness in the future trajectories, \ie after $S_t$ and $A_t$.\footnote{The law of total variance: $\Var[f(X,Y)]= \Exp_{X}\Var_{Y|X}[f(X,Y)] + \Var_{X}\Exp_{Y|X}[f(X,Y)]$~\cite{chung2001course}.} % which affects the Monte Carlo observations of future accumulated costs.
Let us measure the size of these three terms by their trace and define
%define the size of these three terms as 
\begin{align}
\begin{split}
\V_{S_t} \coloneqq \Tr\paren{\Var_{S_t} \Exp_{|S_t} \br{ N_t C_{t:h}}}&, \quad
\V_{A_t|S_t} \coloneqq 	\Tr\paren{\Exp_{S_t} \Var_{A_t|S_t} \br{  N_t \Exp_{|S_t, A_t} \br { C_{t:h}} }},\\
\V_{|S_t, A_t} &\coloneqq \Tr\paren{\Exp_{S_t, A_t} \Var_{|S_t, A_t} \br {N_t C_{t:h}}}.
\end{split}
\end{align}
Hence,  $\Tr\paren{\Var[G_t]} = \V_{S_t} + \V_{A_t|S_t} + \V_{|S_t, A_t}$.
The following theorem shows the size of each term when the policy is Gaussian, which is commonly the case for problems with continuous actions.
\begin{theorem} \label{th:varaince size}
	Suppose the policy $\pi$ is Gaussian such that $\pi_{S_t}(A_t) = \NN(A_t|\mu_\theta(S_t), \sigma I )$, where $\mu_\theta$ is the mean function, and $\theta$ and $\sigma>0$ are learnable parameters.
	Assume the cost function $c$ is bounded and the Q-function $q^{\pi}(s,a)$ is analytic in $a$.
	Then for small enough $\sigma$, it holds that
	\begin{align*}
	\textstyle
	\V_{S_t} = O(h^2), 
	\quad \V_{A_t|S_t}= O\left(\frac{h^2}{\sigma^4}\right), 
	\quad 
	\V_{|S_t, A_t}= O\left(\frac{h^2}{\sigma^4}\right).
	\end{align*}
\end{theorem}
Here we focus on the effects due to the problem horizon $h$ and the policy variance $\sigma$.  \cref{th:varaince size} shows that, when %the policy becomes closer to deterministic ones 
the stochasticity in policy decreases
(\eg when it passes the initial exploration phase)
%(\eg, when it is near the middle or the end of policy optimization),
the terms $\V_{A_t|S_t}$ and $\V_{|S_t,  A_t}$ will dominate variance in policy gradients. An intuitive explanation to this effect is that, as the policy becomes more deterministic, it becomes harder to compute the derivative through zero-order feedback (\ie accumulated costs). In particular, one can expect that $\V_{|S_t, A_t}$ is likely to be larger than $\V_{A_t|S_t}$ when the variation of $C_{t:h}$ is larger than the variation of $q^{\pi}(S_t,A_t)= \Exp_{|S_t,A_t}[C_{t:h}]$.
After understanding the composition of  $\Var [G_t]$, let us analyze $\Var[\Gsaq_t]$ to see why using Q-function estimates as CVs (in Section~\ref{sec:common cv}) can reduce the variance.\footnote{Discussion on $\Var[\Gsv_t]$ is omitted in that $\Gsv_t$ is subsumed by $\Gsaq_t$.}
%value and Q-function estimates as CVs (in Section~\ref{sec:common cv}) can reduce the variance. 
%(We omit $\Gsv_t$ \eqref{eq:state cv} in the analysis in that it is subsumed by $\Gsaq_t$ \eqref{eq:state-action cv}.)
%For illustration, $\Gsv_t$ \eqref{eq:state-action cv} is omitted in that it is subsumed by
Akin to the derivation of \eqref{eq:G_t variance}, one can show that $\Var[\Gsaq_t]$ can be written as 
%, the results below apply also to \eqref{eq:state cv}.
%provided that $\Exp_{|S_t}[N_t \wh{Q}_t]$ can be evaluated, 
%
%where redundant subscripts are added to $\Var$ and $\Exp$ to make the integral variable explicit. 
%have the basic form 
%\begin{align} \label{eq:sa-cv}
%%\gsa_t := N_t C_{t:h} - \paren{
%N_t Q_t - \Exp_{A_t|S_t} \br{N_t Q_t}, \quad Q_t := q_t(S_t, A_t)
%\end{align}
%whose variance can be derived, similar to \eqref{eq:G_t variance}, as
%\begin{small}
\begin{align} \label{eq:sa-cv variance}
%&\eqsp \Var \br{\gsa_t}
%\nonumber\\&=  \Var_{S_t} \br{ \Exp_{|S_t} \br{ N_t C_{t:h}}} 
%+ \Exp_{S_t} \br{ \Var_{|S_t} \br{ N_t \paren{C_{t:h} - Q_t}}}
%\nonumber\\&=  
%\hspace{-4.5mm}
\Var_{S_t}  \Exp_{|S_t} [ N_t C_{t:h}] +
%%\nonumber\\&\eqsp
\Exp_{S_t} \Var_{A_t|S_t} [  N_t (\Exp_{|S_t, A_t} [ C_{t:h}] - \wh{Q}_t)] +
\Exp_{S_t, A_t} \Var_{|S_t, A_t} [N_t C_{t:h} ].
\end{align}
Comparing \eqref{eq:G_t variance} and \eqref{eq:sa-cv variance}, we can see that the CVs in the literature have been focusing on reducing \emph{the second term} $\Var_{A_t|S_t}$. 
Apparently, from the decomposition \eqref{eq:sa-cv variance}, the optimal choice of the state-action CV $\wh{q}$ is the Q-function of the current policy $q^\pi$, because %\footnote{We recall again that the state definition contains time information.} 
$q^\pi(S_t,A_t) := \Exp_{|S_t, A_t} [C_{t:h}]$,
which explains why $\wh{q}$ can be constructed by policy evaluation. When $\wh{q} = q^\pi$, the effect of  $\Var_{A_t|S_t}$ can be completely removed.
In practice, $\wh{q}$ is never perfect (let alone the state-dependent version); nonetheless, improvement in learning speed has been consistently reported.

However,  \cref{th:varaince size} suggests that $\Var_{|S_t, A_t}$ is a similar same size as $\Var_{A_t|S_t}$, implying that even when we completely remove the second term $\Var_{A_t|S_t}$, the variance of the gradient estimate can still be significant.
Indeed, recently~\citet{tucker2018mirage} empirically analyzed the three variance components in \eqref{eq:sa-cv variance} in LQG and simulated robot locomotion tasks. They found that the third term $\Var_{|S_t, A_t}$ is close to the second term $\Var_{A_t|S_t}$, and both of them are several orders of magnitude larger than the first term $\Var_{S_t}$.
%O
Our \cref{th:varaince size} supports their finding and implies that there is a potential for improvement by reducing $\Var_{|S_t, A_t}$. We discuss exactly how to do this next.
%

%In this work, striving for further improvement, we present CVs that are not restrictive to $\Var_{A_t|S_t}$, but attack $\Var_{|S_t, A_t}$ also. 

%the state-action value function at step $t$ defined as
%\begin{align} \label{eq:q function}
%q_t^\star(s,a) := \Exp_{|S_t = s, A_t = a}  \br{C_{t:h} }
%\end{align}
%well, and it vanishes when the estimation is exact. 
%the state-action value function $s, a \mapsto \Exp_{|S_t = s, A_t = a}  \br{C_{t:h} }$, in which case, the variance over $A_t$ (the second term in \eqref{eq:sa-cv variance})
%
%the optimal choice of the state-action-dependent function $q_t$ in \eqref{eq:sa-cv}  is the state-action value function $s, a \mapsto \Exp_{|S_t = s, A_t = a}  \br{C_{t:h} }$, in which case, the variance over $A_t$ (the second term in \eqref{eq:sa-cv variance}) \emph{vanishes}, \cf, \eqref{eq:G_t variance}. 

%===============================================================================
%\vspace{-2mm}
\section{Trajectory-wise Control Variates}
\label{sec:trajectory}
%\vspace{-2mm}

%Intuitively, if we have complete knowledge about 
%%about the distribution of $A_k$ ($k > t$) conditioned on $S_k$, \ie , 
%$\pi$, 
%\emph{the variance on $A_k$ is amenable to suppression}, and ideally
%we should be able to only pay for the variance on random variables  that is really out of our control. 

We propose a new family of trajectory-wise CVs, called TrajCV, that improves upon existing state or state-action CV techniques by tackling \emph{additionally} $\Var_{|S_t, A_t}$, the variance due to randomness in trajectory after step $t$ 
(cf. \cref{sec:why we need new cvs}).
While this idea sounds intuitively pleasing, a technical challenge immediately arises. Recall in designing CVs, we need to know the expectation of the proposed CV function over the randomness that we wish to reduce (see \eqref{eq:estimate}). 
In this case, suppose we propose a CV $g(S_t, A_t, \dots,A_{h-1}, S_{h})$, we would need to know its conditional expectation $\Exp_{|S_t, A_t}[g(S_t, A_t, \dots, S_{h}, A_h)]$.
This need makes reducing $\Var _{|S_t, A_t}$ fundamentally different from reducing $\Var_{A_t|S_t}$, the latter of which has been the main focus in the literature: Because the  dynamics $d$ is unknown, we do not have access to the distribution of trajectories after step $t$ and therefore cannot compute $\Exp_{|S_t, A_t}$; by contrast, reducing $\Var_{A_t|S_t}$ only requires knowing the policy $\pi$.

%although we have access to $\pi$, the dynamics $d$ is unknown and the actions and states are interleaved.
%
At first glance this seems like an impossible quest. But we will show that by a clever divide-and-conquer trick, an unbiased CV can actually be devised to reduce the variance $\Var_{|S_t, A_t}$.
The main idea is to %to contend with this difficulty is to
\textit{1)} decompose $\Var_{|S_t, A_t}$ through repeatedly invoking the law of total variance and then \textit{2)} attack the terms that are \emph{amenable} to reduction using CVs. 
As expected, the future variance cannot be completely reduced, because of the  unknown dynamics. But we should be able to reduce the randomness due to known distributions, namely, the future uses of policy $\pi$.
\begin{figure}[t!]
	\centering

	\begin{subfigure}{.49\textwidth}
		\centering		
		\includegraphics[trim={6cm 2cm 6cm 9cm}, clip, width=0.75\textwidth]{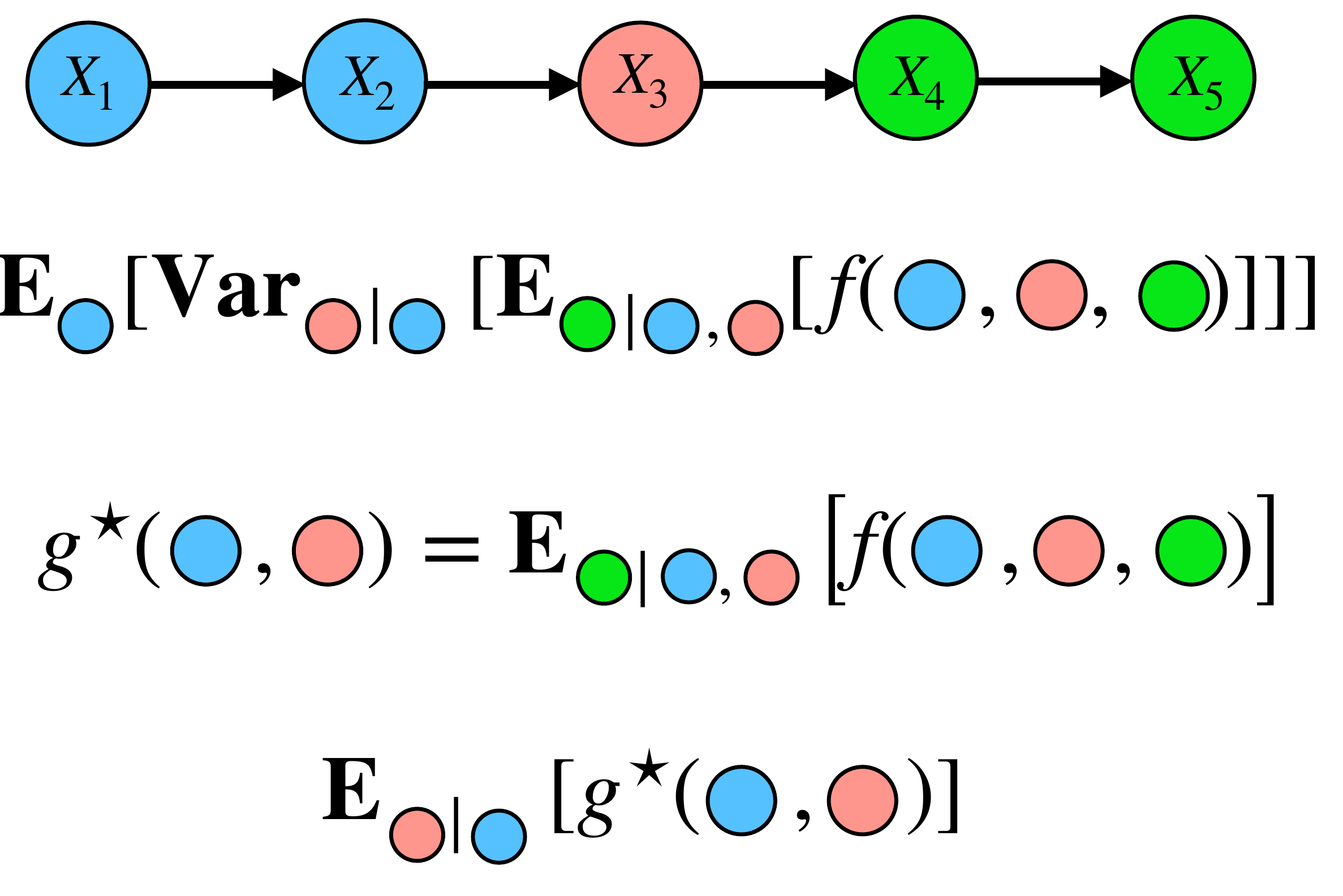}
		\caption{CV on a toy problem}
		\label{fig:cv}
	\end{subfigure}
	\begin{subfigure}{.49\textwidth}
		\centering		
		\includegraphics[trim={5.5cm 6cm 1cm 6cm}, clip, width=1.0\textwidth]{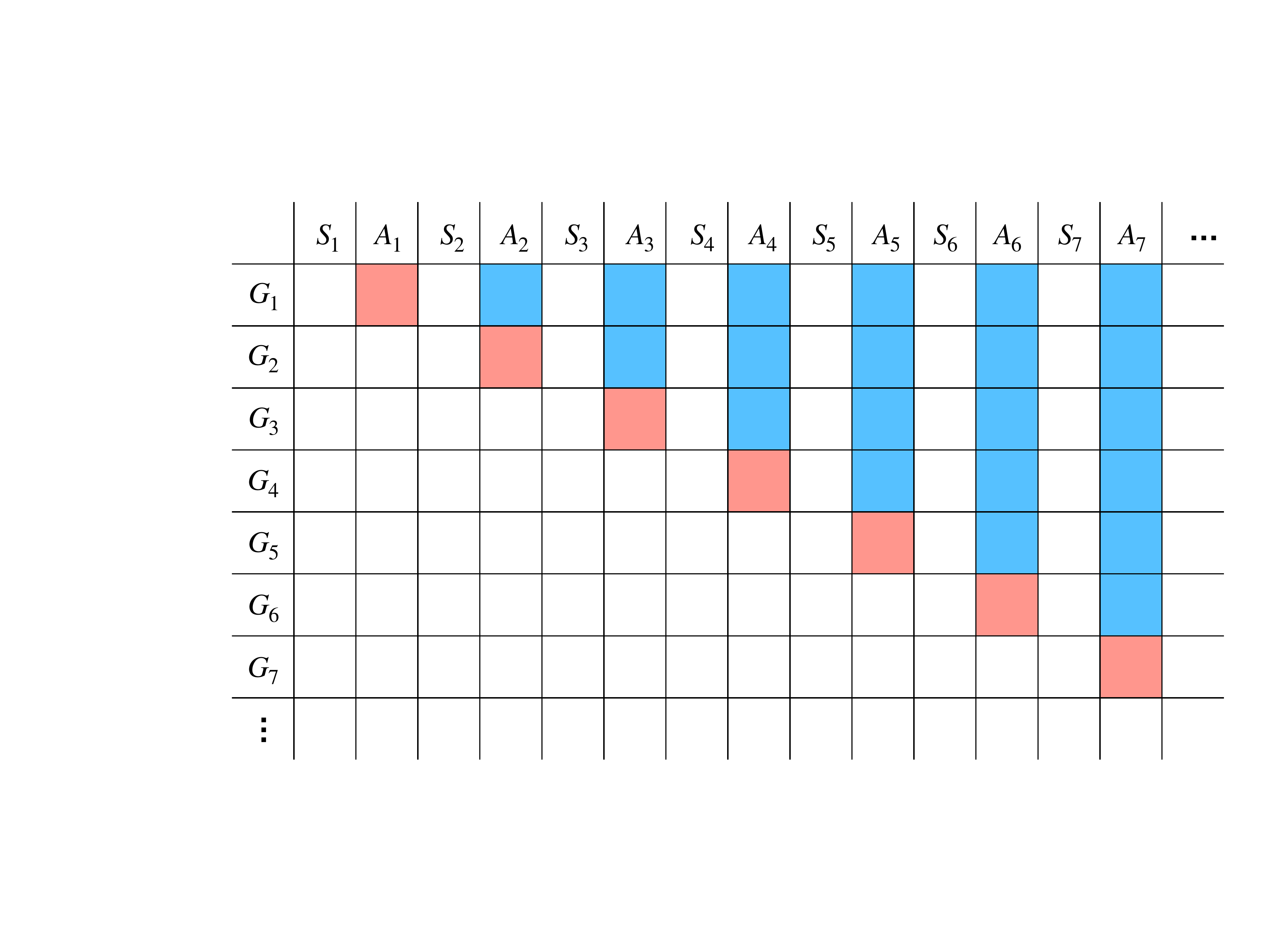}
		\caption{effects of state-action CV and TrajCV}
		\label{fig:comparison}
	\end{subfigure}
	\caption{
		(a) An illustration of the Divide-and-Conquer strategy on the toy problem in Section~\ref{sec:toy example} using color encoding. 
		First row: ordering and coloring of the random variables.
		Second row:  the variance to reduce, \ie the variance due to the randomness of variables in red circle.
		Third row: the corresponding optimal CV.
		Last row: the expectation required to compute the difference estimator for the optimal CV.
		(b) An illustration of the effect of state-action CV \eqref{eq:state-action cv} and TrajCV \eqref{eq:traj-wise difference estimator} on each policy gradient component $G_t$. Each row corresponds to each $G_t$, and each column corresponds to one term in the total variance of $G_t$. 
		For example, the $A_2$ column is associated with the term 
		$\Exp_{S_{1,2}, A_1} \Var_{A_2 |S_2} \Exp_{| S_2, A_2} \br{G_t}$.
		%A visualized comparison between TrajCV and state-action CVs
		%, where  the law of total variance is invoked in the natural order in \eqref{eq:natural ordering}.
		%Each column corresponds to a random variable. 
		The state-action CVs influence the terms in red, whereas TrajCV \emph{additionally} affect the terms in blue.
	}
	\label{fig:some visualization}
\end{figure}

%\vspace{-2mm}
\subsection{A Divide-and-Conquer Strategy} \label{sec:toy example}
%\vspace{-2mm}

Before giving the details, let us first elucidate our idea using a toy problem. Consider estimating $\Exp \br{f(X_{1,5})}$, the expectation of a function $f$ of $5$ random variables, where the subscript $_{i,j}$ denotes the collection of random variables, \ie $X_{1,5} = \{X_1, \dots, X_5\}$.
We can apply the law of total variance repeatedly, in the order indicated by the subscript, and decompose the variance into
\begin{align}
\textstyle
\Var \br{f(X_{1,5})} = \sum_{k=1}^{5} \Exp_{X_{1,k-1}} \Var_{X_k|X_{1,k-1}} \Exp_{X_{k+1, n}|X_{1,k}}[f(X_{1,5})]
\end{align}
%which is depicted graphically in Figure \red{\ref{}}. 
%Akin to the CV \eqref{eq:sa-cv} for $\Var_{A_t|S_t}$ being a function of $A_t, S_t$, 
For example, suppose we wish to reduce $\Var_{X_3|X_{1,2}}$ we simply need to consider a CV in the form $g \paren{X_{1,3}}$, which does not depends on random variables with larger indices. With the difference estimator $f(X_{1,5}) - g \paren{X_{1,3}} + \Exp_{X_3|X_{1,2}} [g \paren{X_{1,3}}]$, the variance $\Var_{X_3|X_{1,2}}$ changes into
\begin{align*}
\Exp_{X_{1,2}} \Var_{X_3|X_{1,2}}[ \Exp_{X_{4,5}|X_{1,3}}[f(X_{1,5})] - g(X_{1, 3})]
\end{align*}
Apparently when $g$ is optimally chosen as
$
g^\star (X_{1,3}) :=  
\Exp_{X_{4, 5}|X_{1,3}}[f(X_{1,5})]
$, this term vanishes.

\textbf{Fact 1\quad }A key property of designing CVs by the recursive decomposition above is that 
%the CV $g(X_{1,3})$ is that 
the inclusion of the extra term, \eg $g \paren{X_{1, 3}} - \Exp_{X_3 | X_{1,2}} \br{g(X_{1, 3})}$, in the difference estimator only affects a single component $\Var_{X_3|X_{1,2}}$ in the total variance, \emph{without influencing the other terms}.
This separation property hence allows for a divide-and-conquer strategy: %namely, 
we can design CVs for each term separately and then combine them; the reduction on each term will add up and reduce the total variance.

\textbf{Fact 2\quad}There is still one missing piece before we can adopt the above idea to design CVs for estimating policy gradients: the ordering of random variables.
In the example above, we need to know $\Exp_{X_3|X_{1,2}} \br{g(X_{1,3})}$ to compute the difference estimator. Namely, it implicitly assumes the knowledge about %the distribution of $X_3$ conditioned on $X_{1,2}$, \ie 
$p(X_3 |X_{1,2})$, which may or may not be accessible. 
Suppose $p(X_3|X_{1,2})$ is not available but $p(X_3|X_{4,5})$ is. We can consider instead using
the law of total variance in a different order, \eg %, in contrast to the order indicated by the subscript,
$
X_4 \rightarrow X_5 \rightarrow X_3 \rightarrow X_1 \rightarrow  X_2
$, 
and utilize the information $p(X_3|X_{4,5})$ to construct a difference estimator to reduce
$\Var_{X_3|X_{4,5}}$.
Therefore, the design of CVs hinges also on the information available. Recall that we only know about the policy but not the dynamics in RL.
% the ordering of the random variables in applying the law of total variance hinges on  

%In words, in order to devise a CV that reduces the variance due to $X_k$, 
%we seek for a function $g$ of $X_{1,k}$ that approximates the expectation of $f$ over the rest of variables, $X_{k+1, n}$,
%and are required to compute the conditional expectation of $g$ over $X_k$ given $X_{1,k-1}$.
%This pattern of constructing CV that aims at the variance on a specific random variable is graphically illustrated in Figure \ref{fig: cv}. 

%\vspace{-1mm}
\subsection{Design of TrajCV}
%\vspace{-1mm}
 %Trajectory-wise Control Variates}
After fleshing out the idea in the example above, we are now ready to present %trajectory-wise CVs 
TrajCVs
for policy gradient. 
Again let us focus on a policy gradient component $G_t$ for transparency. 
%estimating $\Exp \br{G_t}$ for transparency.
%One component of the Monte Carlo policy gradient, 
Recall that $G_t$ is a function of the random variables $S_{t,h}$ and $A_{t,h}$.
Given the information we know about the random variables
(\ie the policy) %$A_k \sim \pi_{S_k}, k \ge t$) 
and the Markovian structure depicted in the Bayes network in Figure~\ref{fig: bayes net 1}, 
a natural ordering of them for applying law of total variance is
\begin{align} \label{eq:natural ordering}
S_t \rightarrow A_t \rightarrow S_{t+1} \rightarrow A_{t+1} \rightarrow \dots \rightarrow S_{h} \rightarrow A_{h}.
\end{align}
Thanks to the Markovian structure, 
 $\Var_{A_k|S_{1,k}, A_{1,k-1}}$ can be simplified into $\Var_{A_k|S_k}$. 
Thus, for reducing $\Var_{A_k|S_k}$, we may consider a CV, $\wh{g}^{(t)}(S_k, A_k)$, where the superscript $^{(t)}$ is a reminder of estimating $\Exp\br{G_t}$. 
With the insights from \cref{sec:toy example}, we see the optimal choice of $\wh{g}^{(t)}$ is
\begin{align*}
\wh{g}^{(t)\star}(S_{k}, A_{k}) 
&= \Exp_{S_{k+1, h}, A_{k+1,h} | S_{k}, A_{k} }  \br{N_t C_{t:h}}
%= N_t  \paren{C_{t:k-1} + \Exp_{S_{k+1, h}, A_{k+1,h} | S_k, A_k }  \br{C_{k:h}}}
%\\
= N_t \paren{C_{t:k-1} + q^\pi(S_k, A_k)},
\end{align*}
%where we recall $q^\pi(S_k, A_k)$ is the Q-function at $(S_k, A_k)$ of time $k$ (the state already contains time information).
% is the state-action value function at step $k$ defined in  \eqref{eq:q function}.
suggesting practically we can use
$\wh{g}^{(t)}(S_k, A_k) := N_t(C_{t:k-1} + \wh Q_k)$, 
where $\wh Q_k: = \wh q(S_k, A_k)$ and  $\wh q \approx q^\pi$
as was in \eqref{eq:state-action cv}.
In other words, we showed that finding the optimal CV for reducing $\Var_{A_k|S_k}$ can be reduced to learning the Q-function; %of policy $\pi$,
this enables us   to take advantage of existing policy evaluation algorithms. 
Now we combine $\{\wh{g}^{(t)}(S_k, A_k)\}_{k=t}^h$ to build the CV for $G_t$. Because from \cref{sec:toy example} these terms do not interfere with each other, we can simply add them together, yielding 
TrajCV: 
%the trajectory-wise CV (TrajCV)
$\sum_{k=t}^h \wh{g}^{(t)}(S_k, A_k)$. 
Equivalently, we have derived a difference estimator:
\begin{align} \label{eq:traj-wise difference estimator}
\Gtraj_t &:= \textstyle G_t - \sum_{t = 1}^h \paren{\wh{g}^{(t)}(S_k, A_k) - \Exp_{A_k|S_k} [\wh{g}^{(t)}(S_k, A_k)]} 
\nonumber \\
&= \textstyle G_t - \sum_{k=1}^h \paren{N_t \wh Q_k - \Exp_{A_k|S_k} [ N_t \wh Q_k]}
\end{align}
%In other words, we showed that finding the optimal CV for reducing $\Var_{A_k|S_k}$ can be reduced to learning the Q-function; %of policy $\pi$,
%this enables us   to take advantage of existing policy evaluation algorithms. 
%
%Now we combine $\{\wh{g}^{(t)}(S_k, A_k)\}_{k=t}^h$ to build the CV for $G_t$. Because from \cref{sec:toy example} these terms do not interfere with each other, we can simply add them together, yielding the trajectory-wise CV: %for estimating $G_t$:
%\begin{align} \label{eq:traj-wise cv}
%\textstyle
%\sum_{k=t}^h \wh{g}^{(t)}(S_k, A_k), \quad \wh{g}^{(t)}(S_k, A_k) := N_t(C_{t:k-1} + \wh{q}(S_k, A_k))
%\end{align}
%where $\wh{q}$ is an approximation of $q^\pi$. Equivalently, we have derived a difference estimator,
%\begin{align} \label{eq:traj-wise difference estimator}
%\hspace{-1mm} \Gtraj_t %&\eqsp N_t C_{t:h} - \sum_{k=t}^h \paren{N_t \paren{C_{t:k-1} + \wh{q}(S_k, A_k)} - 	\Exp_{A_k|S_k} \br{N_t \paren{C_{t:k-1} + \wh{q}(S_k, A_k)}}} \\
%&:= G_t - \Delta^{(t)}_{t:h}, 
%\; \text{where }\Delta^{(t)}_{k} := N_t \wh Q_k - \Exp_{A_k|S_k} [N_t \wh Q_k],  \text{ for } k \ge t.
%\end{align}
%Using the newly introduced notation, 
Comparing TrajCV in \eqref{eq:traj-wise difference estimator} and state-action CV in \eqref{eq:state cv}, we see that the state-action CV only contains the first term in the summation of TrajCV.\footnote{For brevity, we may use CV to mean the difference estimator of that CV when there is no confusion.}
The remaining terms, \ie $N_t (\wh Q_k - \Exp_{A_k|S_k} [\wh Q_k])$, for $t  < k \le h$, can be viewed as multiplying $N_t$ with  estimates of future advantage functions. 
%we see that the previous CV in \eqref{eq:state-action cv} is given as $\Gsaq_t = N_t C_{t:h} - \Delta^{(t)}_{t}$, missing 
%$\Delta_{k:h}^{(t)}$.
%$\Delta_k^{(t)} = N_t (\wh{q} (S_k, A_k) - \Exp_{A_k|S_k} \br{ \wh{q} (S_k, A_k)})$ for $k>t$. 
%These terms $\Delta_k^{(t)}$ for $k>t$ can be viewed as estimates of future advantage functions, because $\wh{q}$ approximates the Q-function.
%With these additional terms, 
Appealing to law of total variance,
$\Var [\Gtraj_t]$ can be decomposed into
%makes $\Var[\wt{G}_t]$ become
\begin{footnotesize}
\begin{align} \label{eq:traj variance}
\hspace{-3mm}
&\Var_{S_t}  \Exp_{|S_t} \br{ N_t C_{t:h}}+
%%\nonumber\\&\eqsp
\Exp_{S_t}  \Var_{A_t|S_t} [  N_t (\Exp_{|S_t, A_t} \br { C_{t:h}} - \wh{Q}_t)] + \\[-1mm]
&\sum_{k=t}^{h}
\Exp_{S_{k}, A_{k}}\Var_{S_{k+1}|S_k, A_k} \br { N_t \Exp_{|S_{k+1}}[ C_{k:h}] }
+ 
\sum_{k=t}^h \Exp_{S_{k+1}}  \Var_{A_{k+1}|S_{k+1} } [ N_t ( \Exp_{|S_{k+1},A_{k+1}}[ C_{k:h}] - \wh{Q}_{k+1} )] \nonumber
\end{align}
\end{footnotesize}%
where we further decompose the effect of $\Var_{|S_t, A_t}$ in the second line
into the randomness in dynamics and actions, respectively. 
Therefore, suppose the underlying dynamics is deterministic (\ie $\Var_{S_{k+1}|S_k, A_k}$ vanishes), and $\wh{q} = q^{\pi}$,
%the Q-function estimation is exact (\ie $\Exp_{|S_{k+1},A_{k+1}}[ C_{k:h}]= \wh{Q}_{k+1}$).
then using TrajCV \eqref{eq:traj-wise difference estimator} 
%the CV in~\eqref{eq:traj-wise cv} 
would completely remove  $\Var_{A_t|S_t}$  and  $\Var_{|S_t, A_t}$, the latter of which previous CVs \eqref{eq:state cv} and \eqref{eq:state-action cv} cannot affect.
In \cref{fig:comparison},
we visualize the effect of TrajCV and state-action CV on each policy gradient component $G_t$, for $1 \le t \le h$:
state-action CVs only influence the diagonal terms, while TrajCVs are able to affect the full upper triangle terms.
%we summarize $\Var[\Gsaq_{1:h}]$ and $\Var[\Gtraj_{1:h}]$ of the full estimator of $G$.
%, when $\tilde{G}_t$ are constructed by \eqref{eq:state cv}, \eqref{eq:state-action cv}, and \eqref{eq:traj-wise cv}.
%The proof is given in \cref{app:proof}. 
Note that in implementation of TrajCV for $G_{1:h}$, we only need to compute quantities 
%$\wh{q}(S_t, A_t)$, $\Exp_{A_t|S_t}\wh{q}(S_t, A_t)$, and $\nabla \Exp_{A_t|S_t}\wh{q}(S_t, A_t)$
$\wh Q_t$,  $\Exp_{A_t|S_t} [ \wh Q_t]$ and 
$\nabla \Exp_{A_t|S_t}[\wh Q_t]$
along a trajectory (done in $O(h)$ time) and they can be used to compute $\{\Gtraj_t\}_{t=1}^h$ \eqref{eq:traj-wise difference estimator}.%
\footnote{
	We have 
	$
	\Exp_{A_t|S_t} [N_t \wh Q_t]= 
	\Exp_{A_t|S_t}  [\nabla \log \pi_{S_t} \wh Q_t] =
	\Exp_{A_t|S_t} \br{\frac{\nabla \pi_{S_t}(A_t)}{\pi_{S_t}(A_t)}\wh Q_t}
	%= \Exp_{A_t|S_t} \br{\nabla \pi_{S_t}(A_t)\wh q(S_t, A_t)}
	= \nabla \Exp_{A_t|S_t} [\wh Q_t]
	$.
	%by the definition of $N_t$.
}
In addition, we remark that when $\wh{q}(s, a) = \wh{v}(s)$, TrajCV    
reduces to the state-dependent CVs.
%This can be verified by calculating the difference estimator using $q_k = q_k^\star$ and 
%\begin{align}
%q_k^\star (S_k, A_k) = c(S_k, A_k) + \Exp_{A_{k+1}|S_{k+1}}\br{q_{k+1}^\star(S_{k+1}, A_{k+1})}
%\end{align}
%We get $\Exp_{A_k |S_k} \br{N_t q_t^\star(S_t, A_t)}$,
%	%\begin{align*}
%%N_t C_{t:h} - \sum_{k=t}^h \paren{ N_t q_k^\star (S_k, A_k) - \Exp_{A_k|S_k} \br{N_t q_k^\star (S_k, A_k)}} = \Exp_{A_k |S_k} \br{N_t q_t^\star(S_t, A_t)}
%%\end{align*}
%which is equal to the expectation of  $G_t$ \cite{}.

%In contrast, 
%the state-action dependent CVs investigated in the recent works are restricted to reducing the variance due to $A_t$~\cite{gu2016q,grathwohl2017backpropagation,tucker2018mirage}. 
%Note that when $t  < k \le h$, it becomes
%\begin{align*}
%N_t \paren{Q_k - V_k}, \quad V_k := v_k(S_k)
%\end{align*}
%where $v_k$ is the state value function at step $k$. 

%\begin{figure}[b]
\begin{figure}[t!]
	%\vspace{-2mm}
	\centering
	\begin{subfigure}{.45\textwidth}
		\centering		
		\includegraphics[trim={4cm 10cm 10cm 6.5cm}, clip, width=0.9\textwidth]{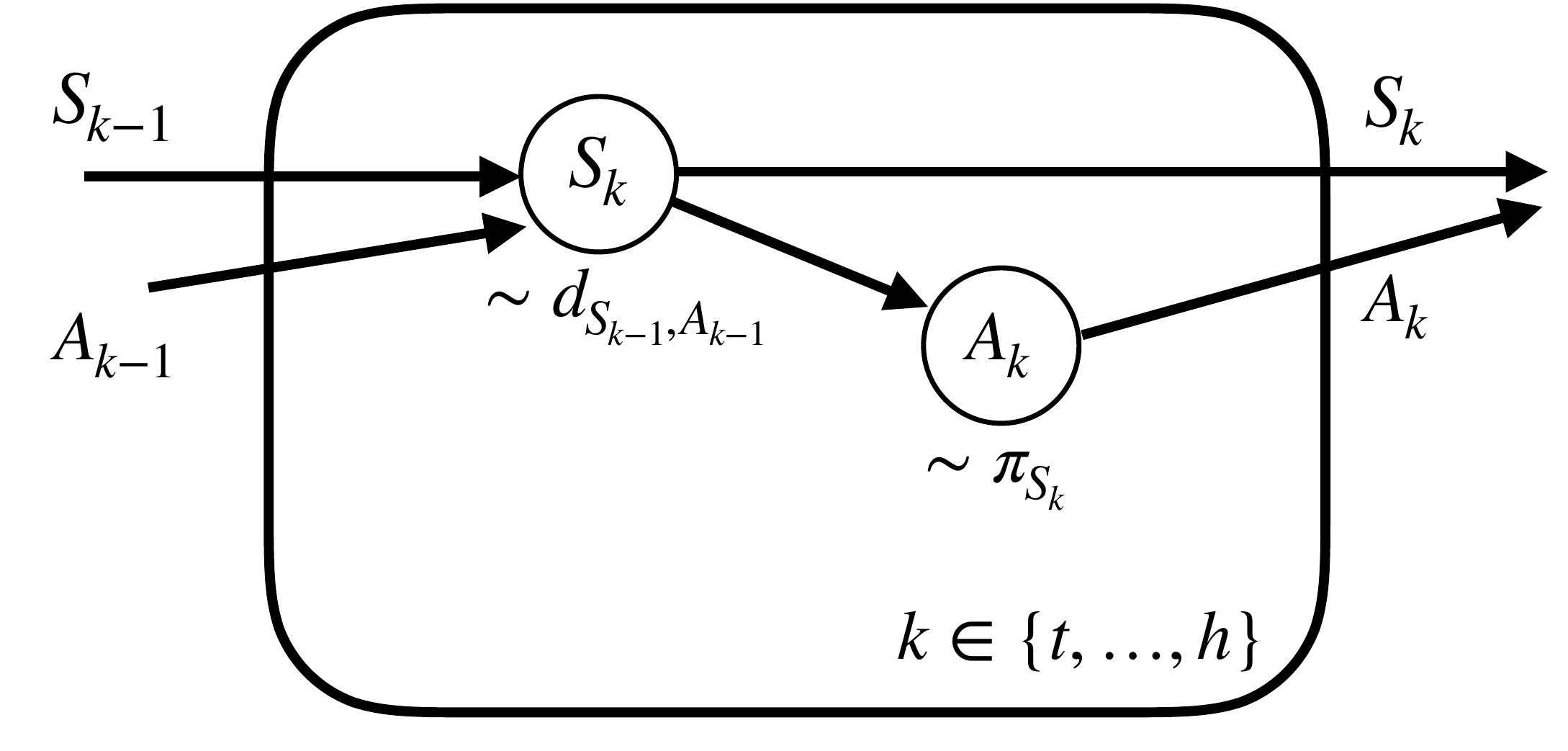}
		\caption{before policy reparameterization}
		\label{fig: bayes net 1}
	\end{subfigure}
	\begin{subfigure}{.45\textwidth} 
		\centering		
		\includegraphics[trim={4cm 10cm 10cm 6.5cm}, clip, width=0.9\textwidth]{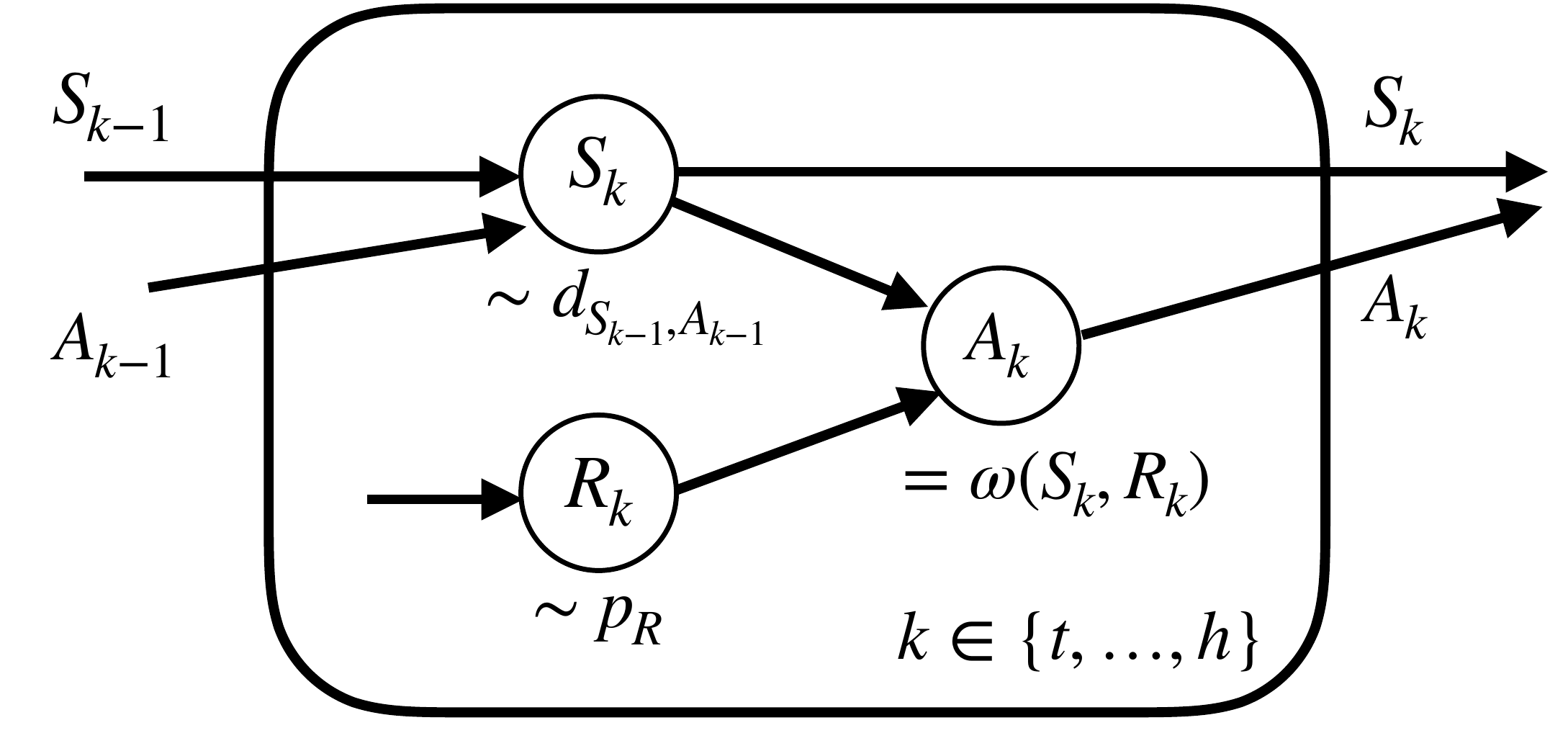}
		\caption{after policy reparameterization}
		\label{fig: bayes net 2}
	\end{subfigure}
	\caption{
		Bayes networks for the random variables in $G_t$ \eqref{eq:policy gradient}, before and after reparameterization. After policy is reparameterized, action $A_k$ is decided by state $S_k$ and action randomness $R_k$.
	}
\end{figure}

%\vspace{-1mm}
\subsection{The Natural Ordering in \eqref{eq:natural ordering} is Optimal}
\label{sec:optimal ordering}
%\vspace{-1mm}

Recall in \cref{sec:toy example} we mentioned that the admissible ordering of random variables used in invoking the law of total variance depends on the information available. Here we show that the chosen ordering \eqref{eq:natural ordering} is indeed, the best ordering to adopt, as we only know the policy, not the dynamics.

We compare \eqref{eq:natural ordering} against potential orderings constructed by reparameterizing the policy such that its randomness in action becomes \emph{independent of the input state}.
We suppose the policy $\pi \in \Pi$ can be reparameterized by a function $\omega: \Sc \times \Rc \to \Ac$ and a distribution $p_R$, so that for all $s \in \Sc$,  $\omega(s, R)$ % R \sim p_R$ 
and $\pi_s$ are equal. 
This is not a restricted assumption. 
For example, 
a Gaussian stochastic policy with mean $\mu_\theta(s)$ and covariance $\Sigma(s)$, widely used in continuous domains~\cite{williams1992simple,schulman2015trust,baxter2001infinite}, can be framed as 
$\omega(s, R) = \mu(s) +\Sigma^{1/2}(s) R $, where $R$ is multivariate standard normal. 
Even when $\Ac$ is discrete, a soft-max layer can be applied after $\omega$ to obtain a Gibbs (or Boltzmann) distribution~\cite{landau1958statistical,sutton2000policy}.

Reparameterization gives rise to the chance of designing a larger family of %trajectory-wise CVs.
TrajCVs.
By extracting the randomness in action selection into the random variable $R_{t,h}$,  the policy gradient component $G_t$ becomes a function of $R_{t,h}$ and $S_{t,h}$, as depicted in 
the Bayes network in 
%The Bayes network of the random variables after policy reparameterization is depicted in 
Figure \ref{fig: bayes net 2}.
When applying the law of total variance,  the ordering the random variables now can have many possibilities. 
For instance, in the extreme case,  the randomness of actions can be ordered before states (except $S_t$) as
\begin{align} \label{eq:action randomness first ordering}
S_t \rightarrow R_{t} \rightarrow \dots \rightarrow R_{h} \rightarrow  S_{t+1}\rightarrow \dots \rightarrow S_h,
\end{align}
leading to a CV
% in the form of $g(R_{t,h})$ 
that's a function of $R_{t,h}$
%whose optimal choice is
bearing the optimal choice
%\begin{align} \label{eq:action randomness first optimal cv}
$\Exp_{S_{t+1, h} |R_{t,h}, S_t} \br {N_t C_{t:h}}$.
%\end{align}
Note that $\Exp_{S_{t+1, h} |R_{t,h}, S_t} \br {N_t C_{t:h}}$ is a function that inputs the observable action randomness $R_{t,h}$, not the randomness of the unknown dynamics. 
Therefore, it can be approximated, \eg, if we have a biased simulator of the dynamics.\footnote{We sample all the action randomness $R_{t,h}$ first, execute the policy $\pi$ in simulation with fixed randomness $R_{t,h}$, and then collect the statistics $N_t C_{t:h}$.}
%Given a simulator, which is usually derived from classical mechanics, an estimate of~\eqref{eq:action randomness first optimal cv} can be constructed by Monte Carlo sampling, 
One might ask, given all possible orderings of random variables, which ordering we should pick to design the CV.
%s that lead to different forms of CVs, which 
Interestingly, to this question, the most natural one and the optimal one coincide. The proof is deferred to \cref{app:proof}.
\begin{theorem} \label{thm: optimal ordering}
%Given the conditional independence assumptions on the random variables embedded in the Bayes net (Figure~\ref{fig: bayes net 2}),
Suppose that policy specified by $\omega$ and $p_R$ is known, but the dynamics $d$ is unknown.
Assume the optimal CV of a given ordering of random variables $S_{t,h}$ and $R_{t,h}$ can be obtained.
The the optimal ordering that minimizes the residue variance is the natural ordering~\eqref{eq:natural ordering} .
\end{theorem}

%\paragraph{Trade off}
Theorem~\ref{thm: optimal ordering} tells us that
if the optimal CVs are attainable (\ie we can estimate the Q-function exactly), then the natural ordering is optimal. 
However, in practice, the CVs are almost always suboptimal due to error in estimation. 
If the dynamics is relatively accurate and the computing resources for simulation are abundant, then although the residue is higher, the ordering \eqref{eq:action randomness first ordering} could actually be superior. 
Therefore, the ordering of random variables based on the relative accuracy of different estimates is an interesting practical question to pursue in future work.
In the experiment section of this work, we will focus on the natural ordering \eqref{eq:natural ordering}.

%===============================================================================
%\subsection{Practical Considerations}
%One way of approximating $q_k$ is by learning a on-policy value function $v$ and use a simulate one step
%\begin{align*}
%q_k(s, a) = c(s, a) + \Exp_{s'} \br{v(s')}
%\end{align*}
%if the model is deterministic then only one sample is needed. 
%\paragraph{Computation complexity table}
%For each $G_t$, trajectory-wise CV \eqref{eq:traj-wise cv} have $h-t+1$ terms. 
%Therefore, for all $G_{1,h}$, there are $O(h^2)$ terms. 
%The effect of these $O(h^2)$ terms 
%
%When Monte Carlo is used to estimate $\Exp_{A_t |S_t} \br{N_t Q_t}$.
%
%
%obviate the need for Monte Carlo sampling to estimate $\Exp_{A_t}$

%\vspace{-2mm}
\section{Experimental Results and Discussion}
%\vspace{-2mm}
\label{sec:result}
%\vspace{-1mm}

%\subsection{Experimental Setups}
Although the focus of this paper is the theoretical insights, we illustrate our results with experiments in learning neural network policies to solve the CartPole balancing task in OpenAI Gym~\citep{brockman2016openai} with DART physics engine~\citep{Lee2018}. 
In CartPole, 
the reward function is the indicator function that equals to one when the pole is close to being upright and zero otherwise. 
This is a delayed reward problem in that the effective reward signal is revealed only when the task terminates prematurely before reaching the horizon, \ie when the pole deviates from being upright.
The start state is perturbed from being vertical and still by an offset uniformly sampled from $[-0.01, 0.01]^{d_\Sc}$, and the dynamics is deterministic.\footnote{Symbol $d_\Sc$ denotes the dimension of $\Sc$.}
The action space is continuous and Gaussian policies are considered in the experiments. 
%Here we consider CartPole with three horizons ($1000$, $2000$, and $4000$) with increasing difficulty.
The policy is optimized by natural gradient descent \cite{kakade2002natural} with a KL-divergence safe guard on the policy change to be robust to outliers in data collection. Below we report in rewards, negative of costs, which is the natural performance measure provided in OpenAI Gym.

\begin{figure}[t!]
	\centering
	\begin{subfigure}{.24\textwidth}
		\centering		
		\includegraphics[width=1.0\textwidth]{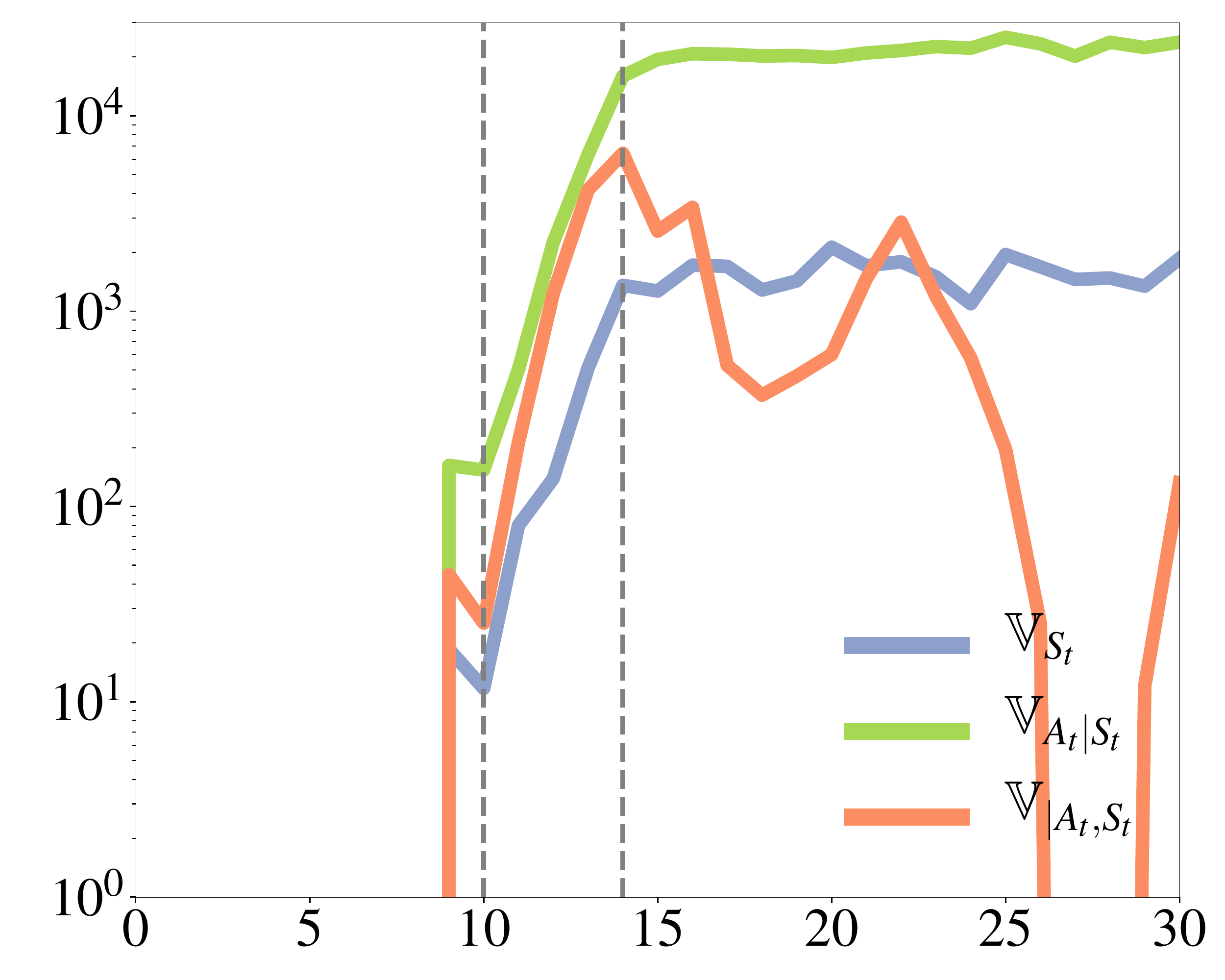}
		\caption{$\sigma = 0.1$}
		\label{fig:std_0-1}
	\end{subfigure}
	\begin{subfigure}{.24\textwidth}
		\centering		
		\includegraphics[width=1.0\textwidth]{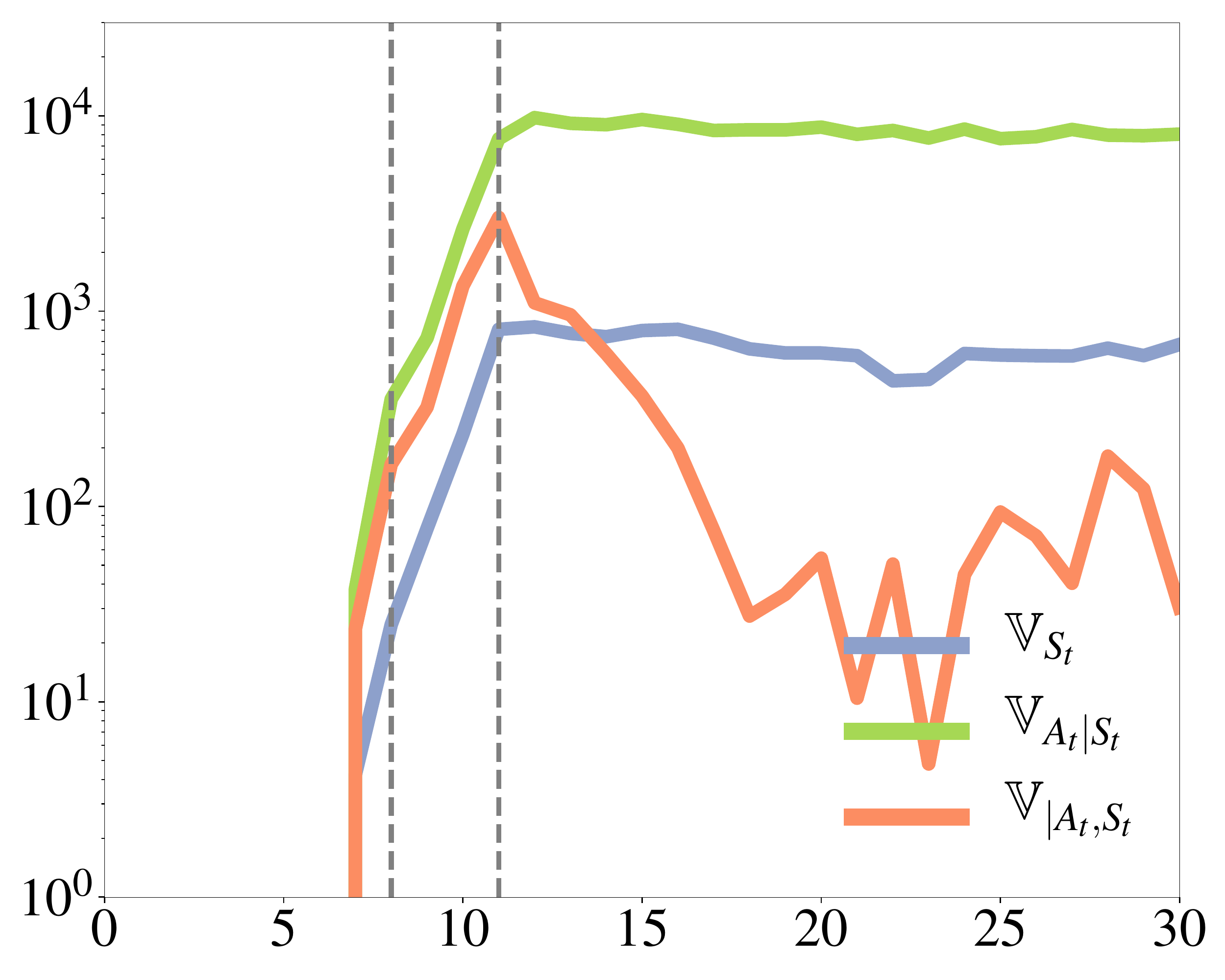}
		\caption{$\sigma = 0.3$}
		\label{fig:std_0-3}
	\end{subfigure}
	\begin{subfigure}{.24\textwidth}
		\centering		
		\includegraphics[width=1.0\textwidth]{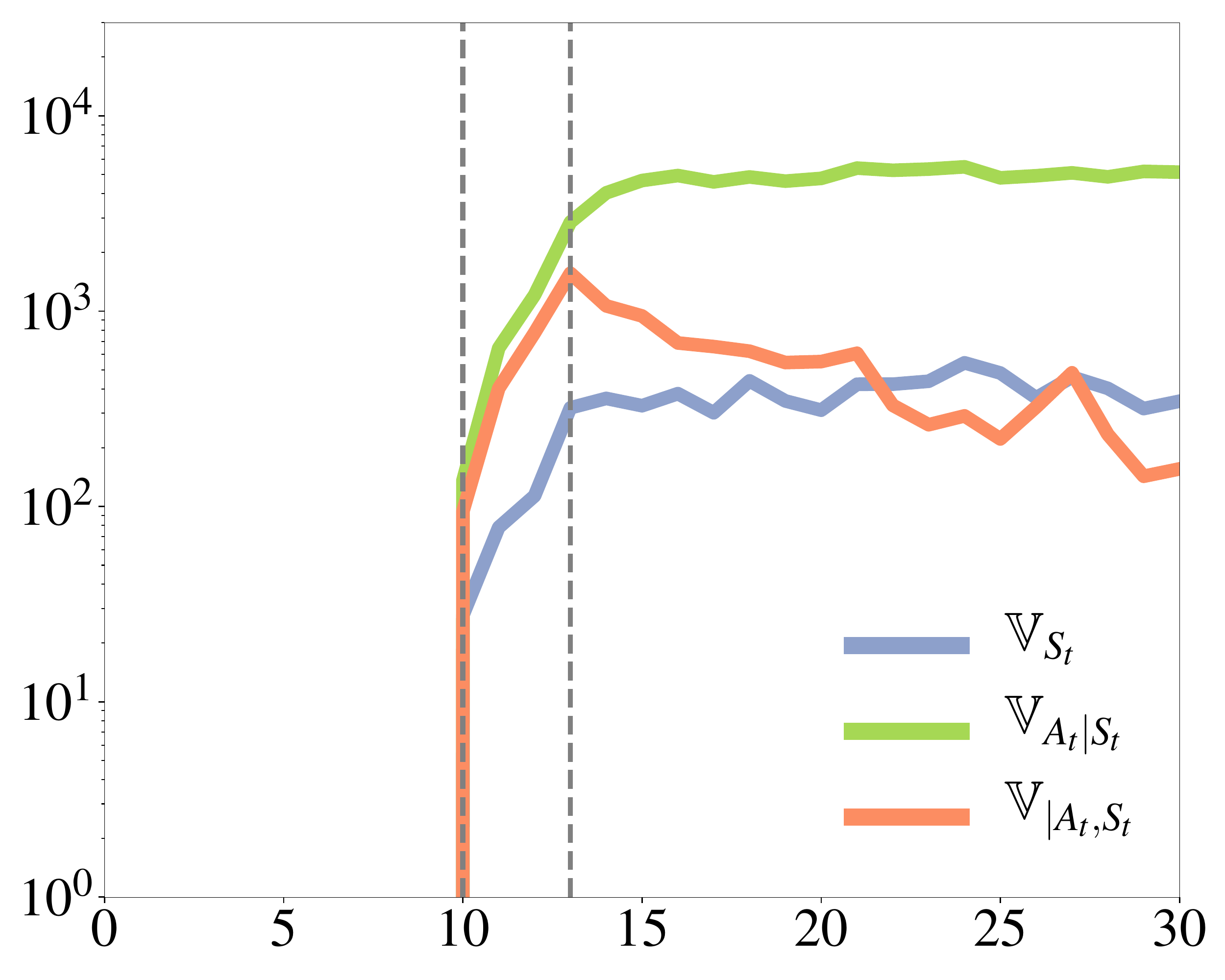}
		\caption{$\sigma = 1.0$}		
		\label{fig:std_1-0}		
	\end{subfigure}
	\begin{subfigure}{.24\textwidth}
		\centering			
		\includegraphics[width=1.0\textwidth]{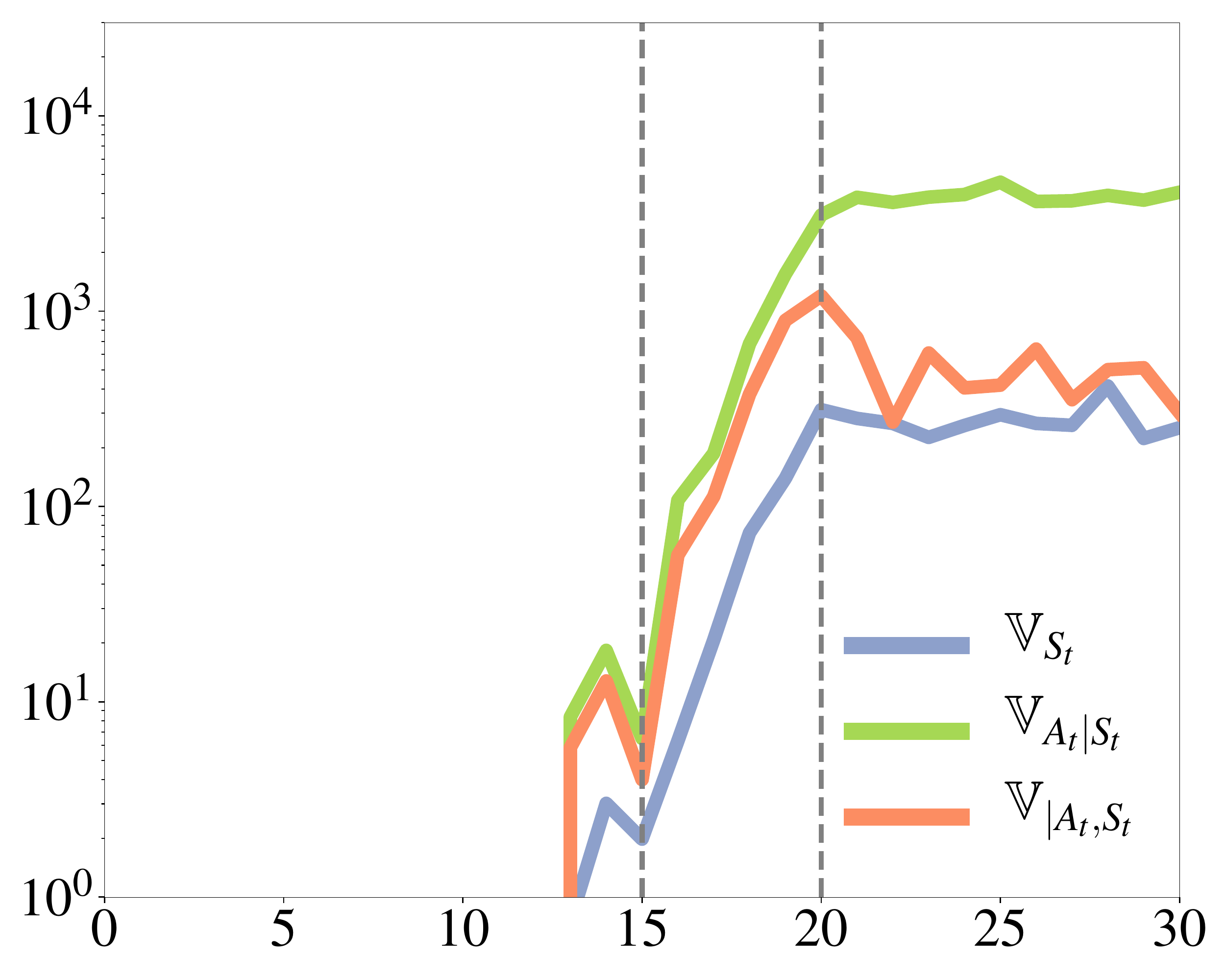}
		\caption{$\sigma = 3.0$}
		\label{fig:std_3-0}		
	\end{subfigure}
	\caption{
		The size of components of $\Tr \paren{\Var[G_t]}$ during training, for $t = 100$, evaluated at policies generated under the ``upper bound'' setting that emulates noiseless estimates for the CartPole problem with horizon $h = 1000$.
		Different initial values for $\sigma$ of the Gaussian policy (defined in \cref{th:varaince size}) are used.
		The $x$ axis denotes the iteration, and the $y$ axis is in $\log$ scale. 
		The two vertical dashed lines mark the boundaries of iterations where the expected accumulated rewards is between $50$ and $900$. 
	}
	\label{fig:sigmas}
	%\vspace{-5mm}
\end{figure}

In \cref{fig:sigmas}, we corroborate the theoretical findings in \cref{th:varaince size} by empirically evaluating the components of 
$\Tr\paren{\Var[G_t]}$ during learning, for $t = 100$. 
The learning process on CartPole can be partitioned into \emph{three stages}:
1) initial exploration, when the policy performs very poorly and improves slowly,
2) rapid improvement, when the policy performance increases steeply, 
3) near convergence, when the policy reaches and stays at the peak performance.
The dashed vertical lines in  \cref{fig:sigmas}  delimit these three stages. 
%More concretely, at the iterations between and on the vertical lines (stage 2),  the average accumulated rewards is between $50$ and $900$. 
In the rapid improvement stage, due to the variance in the accumulated rewards, \ie length of the trajectories, $\V_{|A_t, S_t}$ is large, close to $\V_{A_t|S_t}$ and about 10 times of $\V_{S_t}$
% dominating the variance, 
as predicted by \cref{th:varaince size}.
Furthermore, as $\sigma$ increases, the gap between the peak values of $\V_{S_t}$ and $\V_{A_t|S_t}$ narrows.

\begin{figure}[t]
	\centering
	\begin{subfigure}{.24\textwidth}
		\centering		
		\includegraphics[width=1.0\textwidth]{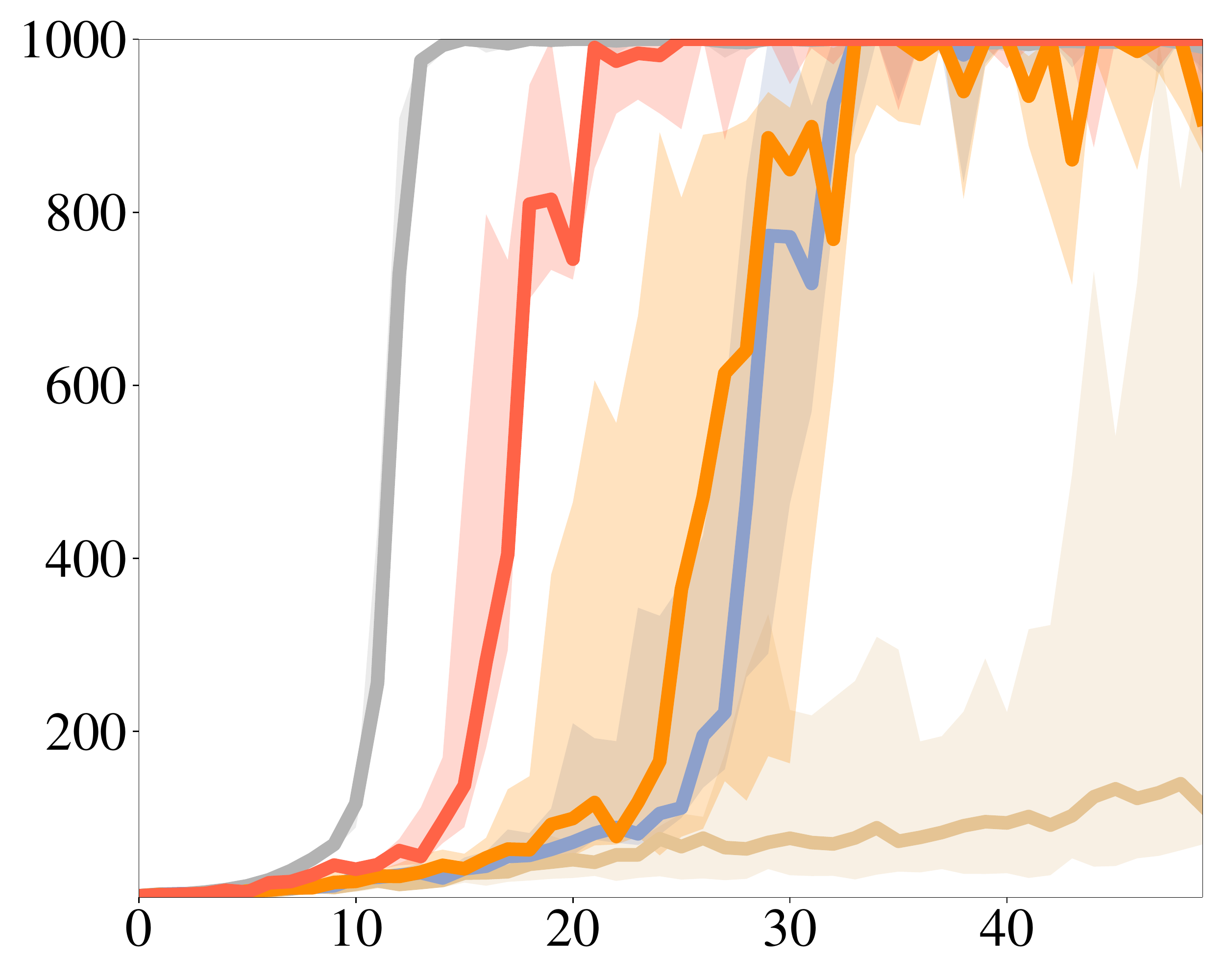}
		\caption{$h = 1000$ (MC)}
		\label{fig:cp1k}
	\end{subfigure}
	\begin{subfigure}{.24\textwidth}
		\centering		
		\includegraphics[width=1.0\textwidth]{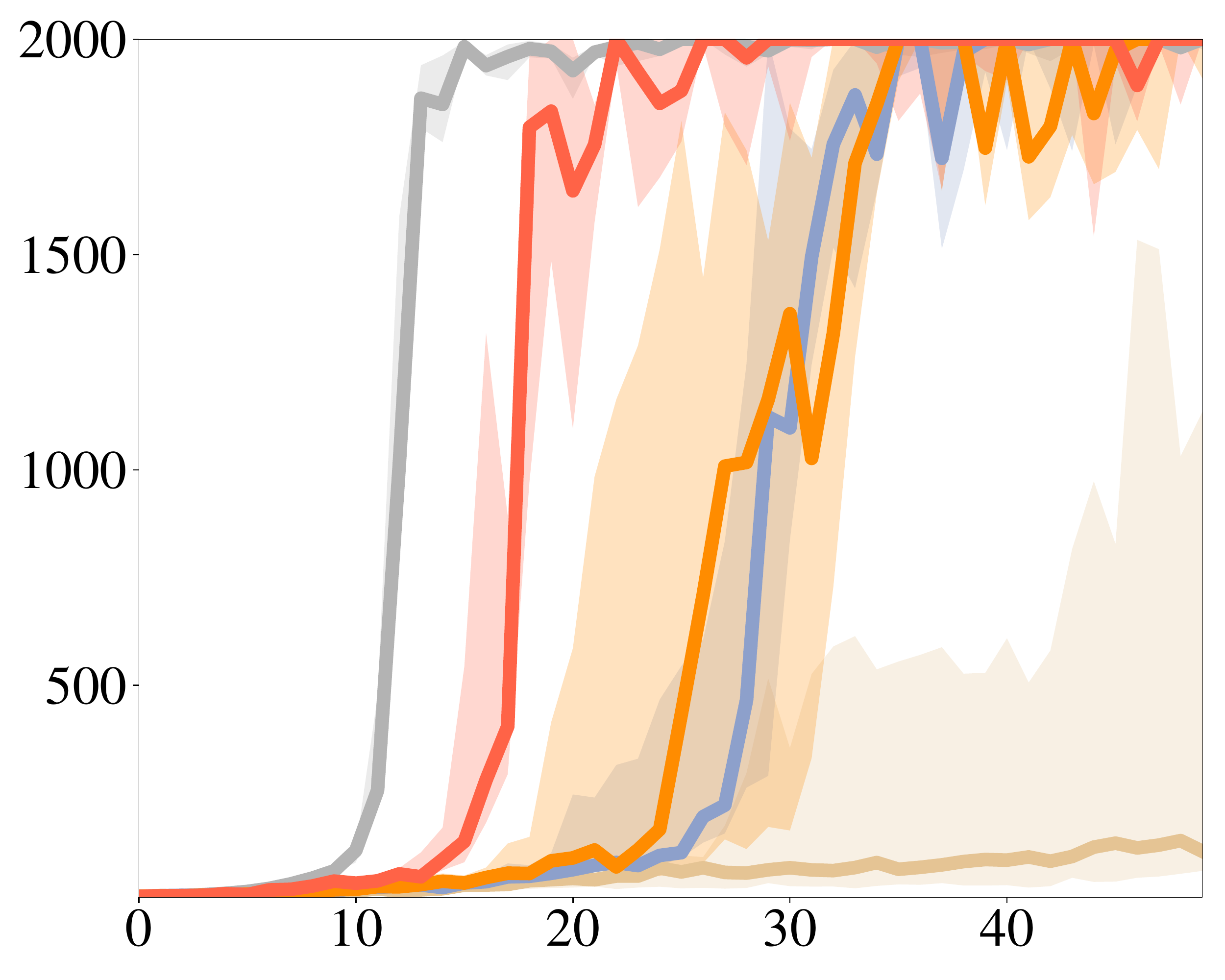}
		\caption{$h = 2000$ (MC)}
		\label{fig:cp2k}
	\end{subfigure}
	\begin{subfigure}{.24\textwidth}
		\centering		
		\includegraphics[ width=1.0\textwidth]{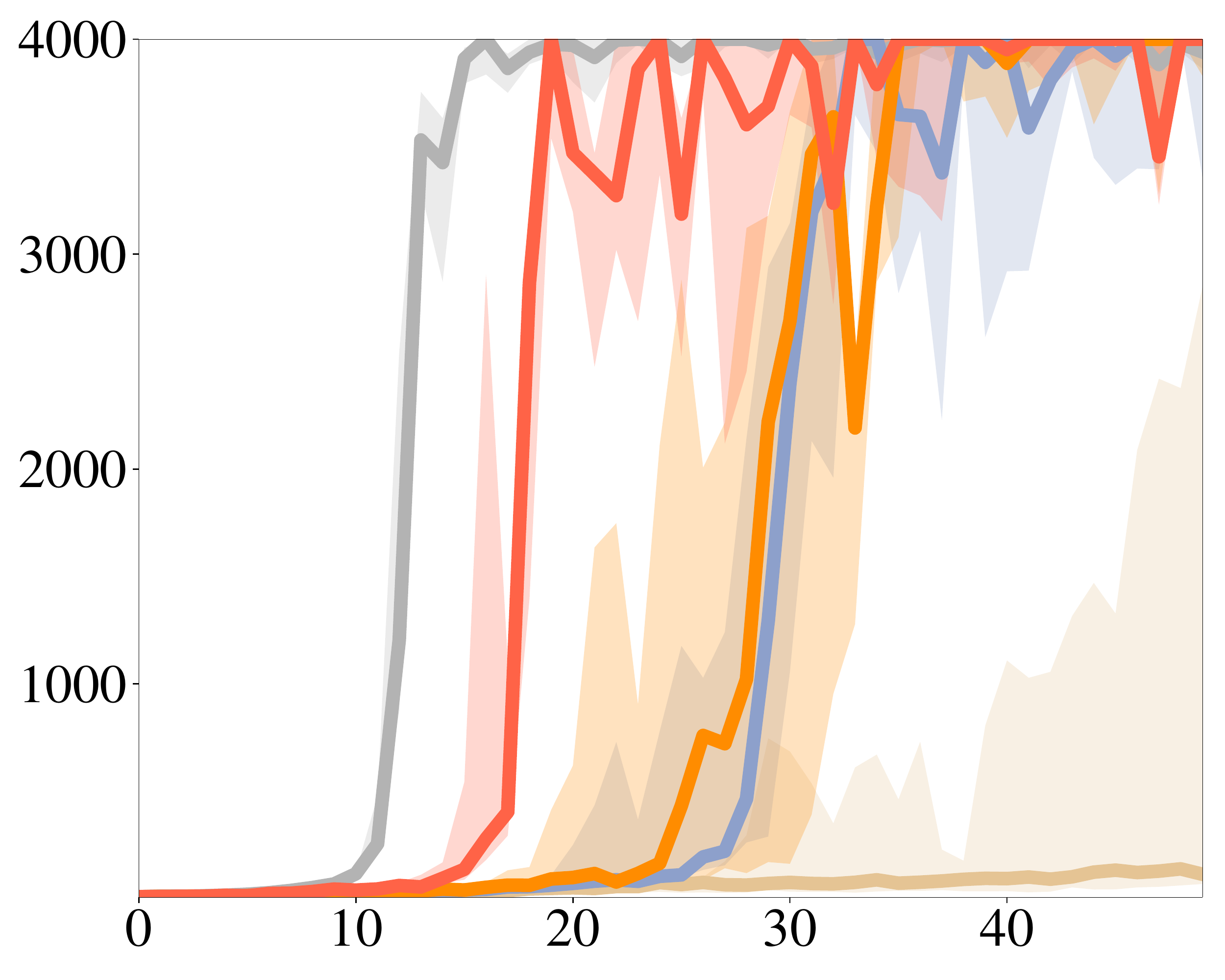}
		\caption{$h = 4000$ (MC)}
		\label{fig:cp4k}
	\end{subfigure}
	\begin{subfigure}{.24\textwidth}
		\centering		
		\includegraphics[ width=0.8\textwidth]{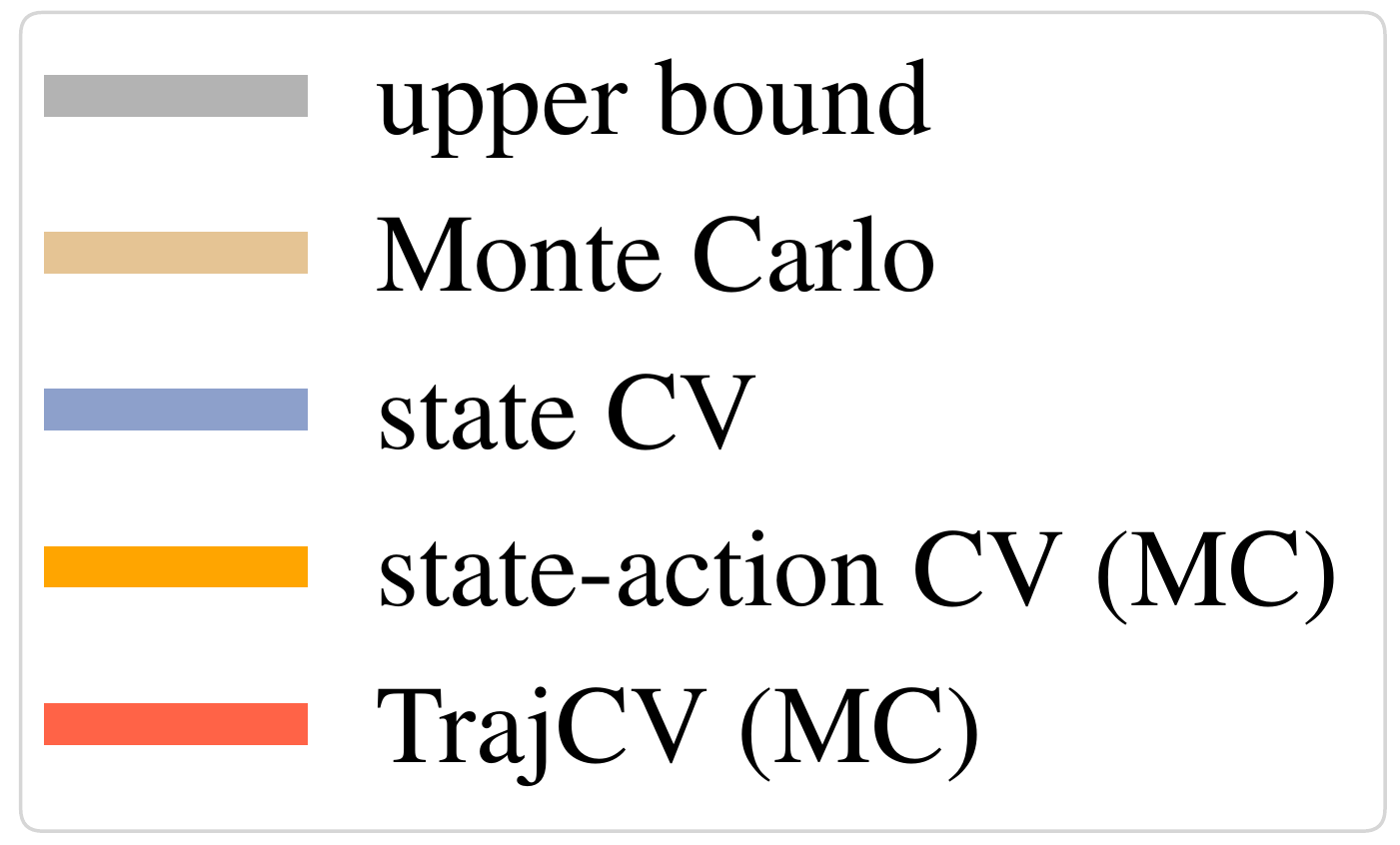}
		\label{fig:cplegend}
	\end{subfigure}
	\caption{
		The  results of naive Monte Carlo estimate of policy gradient, and it augmented with state  CV, state-action CV, and TrajCV on CartPole problems with horizon $h=1000, 2000, 4000$.
		``MC" in the legend indicates that  
		\emph{the expectation over actions for state-action CV and TrajCV is approximated using Monte Carlo (1000 samples)}. 
		Upper bound emulates the results of noiseless estimates. The $x$ and $y$ axes are iteration number and accumulated rewards (the higher the better), respectively. 
		The median of 8 random seeds is plotted, and the shaded area accounts for 25\% percentile. 
	}
	\label{fig:cps}
	%\vspace{-5mm}
\end{figure}

\begin{figure}[t]
	\centering
	\begin{subfigure}{.24\textwidth}
		\centering		
		\includegraphics[width=1.0\textwidth]{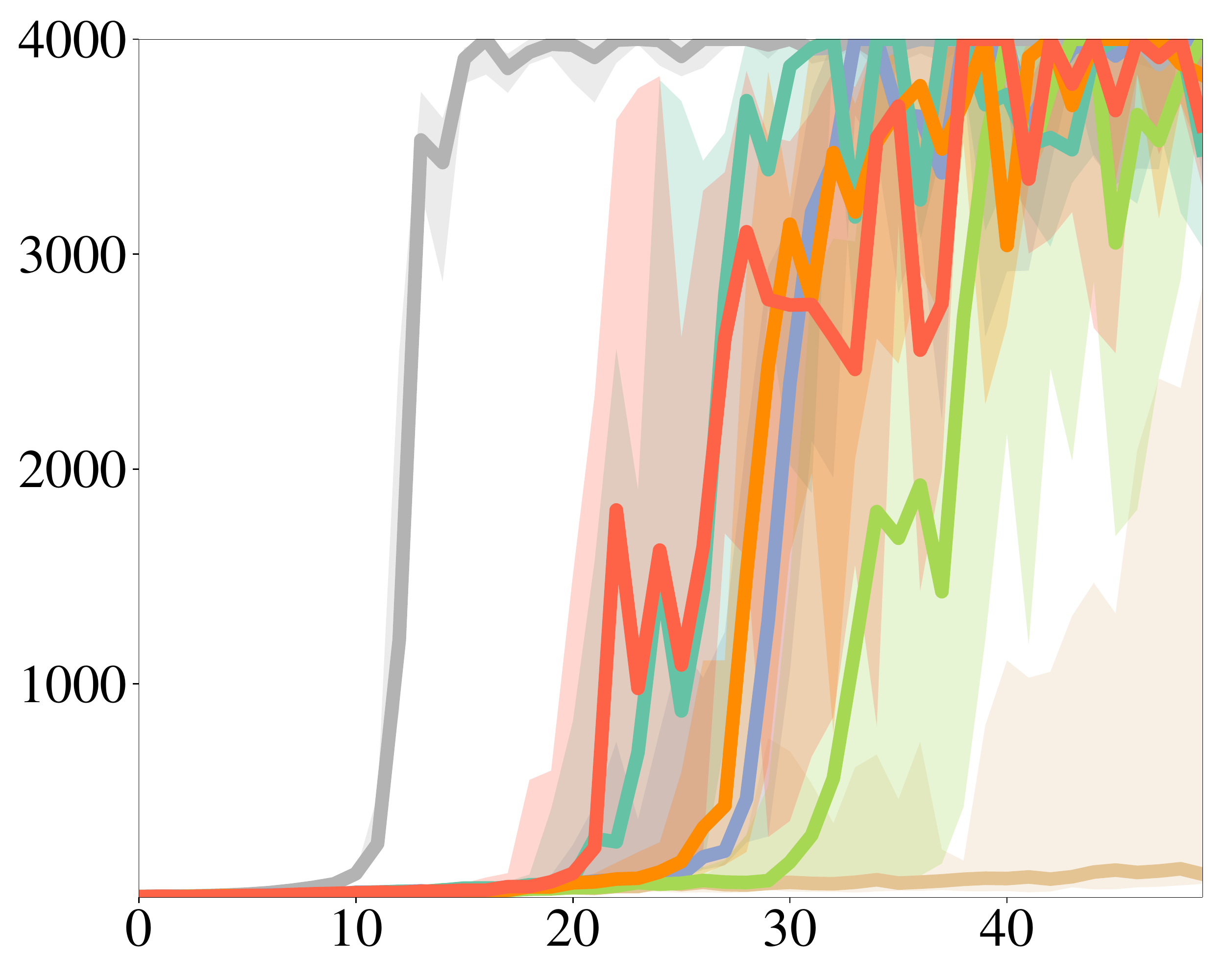}
		\caption{$h = 4000$ (diff)}
		\label{fig:cp4k-diff}
	\end{subfigure}
	\begin{subfigure}{.24\textwidth}
		\centering		
		\vspace{-1em}
		\includegraphics[width=0.88\textwidth]{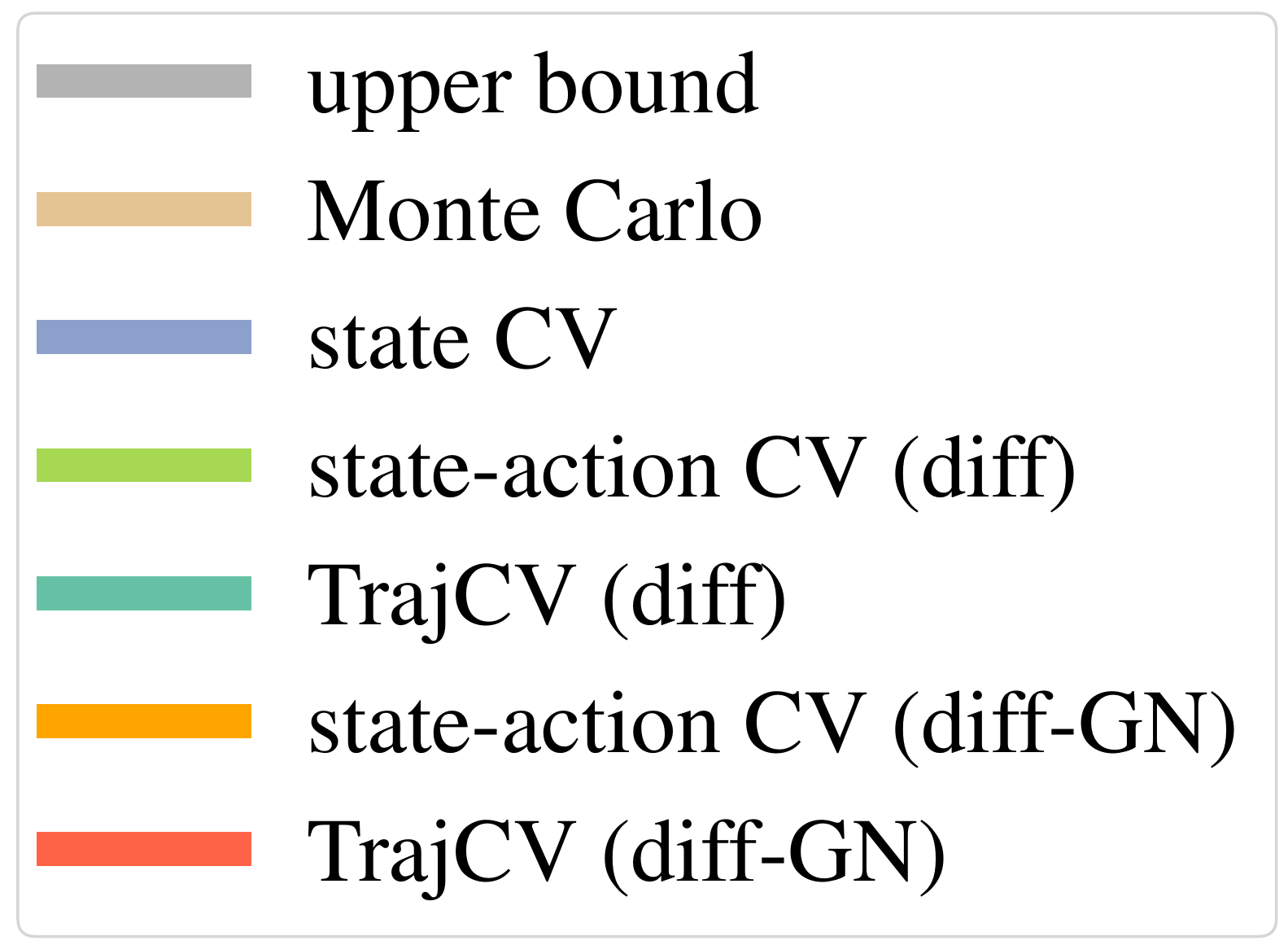}
		%	\caption{CartPole2k}
		\label{fig:cp4k-diff-legend}
	\end{subfigure}
	\begin{subfigure}{.24\textwidth}
		\centering		
		\includegraphics[ width=1.0\textwidth]{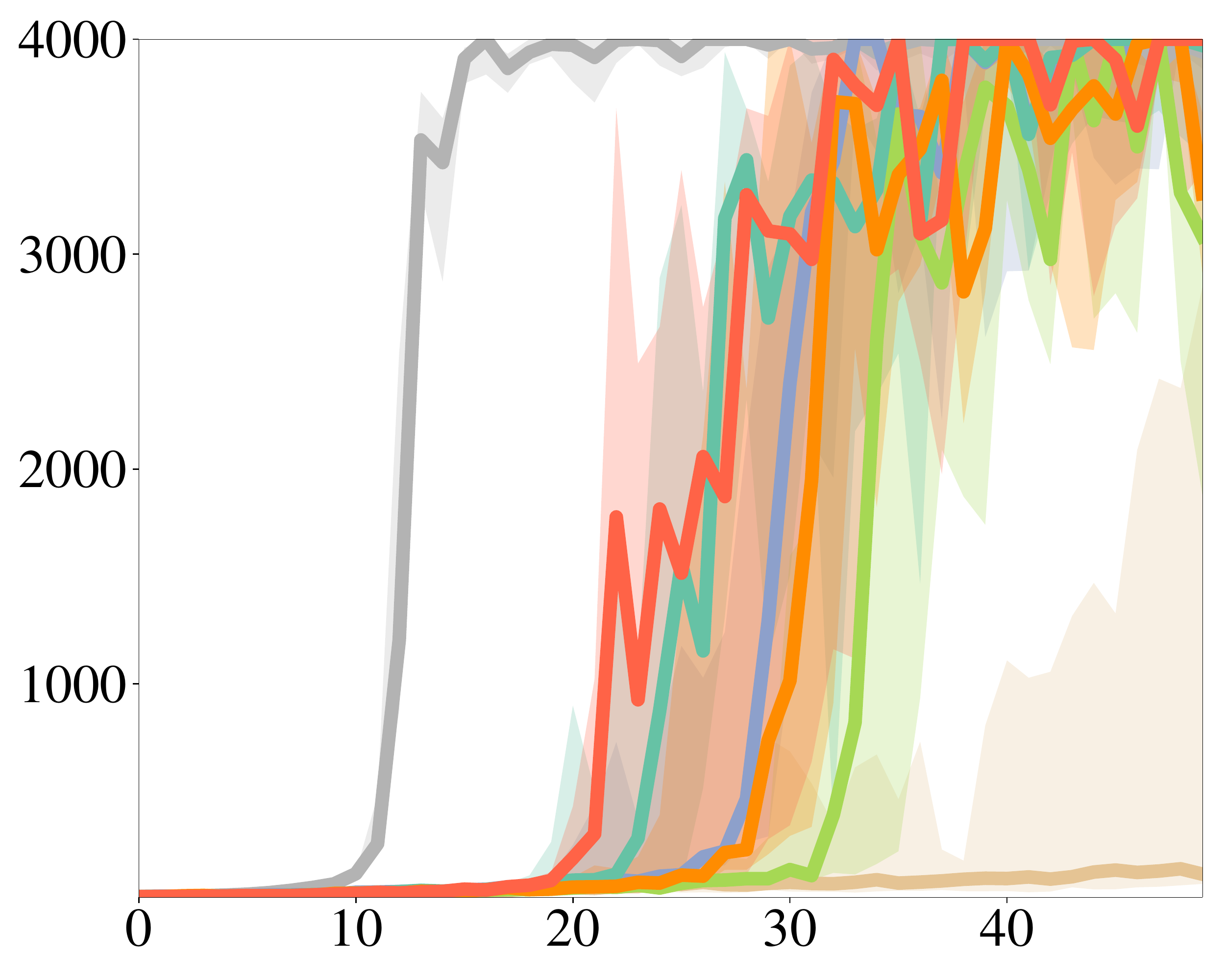}
		\caption{$h = 4000$ (next)}
		\label{fig:cp4k-next}
	\end{subfigure}
	\begin{subfigure}{.24\textwidth}
		\centering		
		\vspace{-1em}		
		\includegraphics[ width=0.88\textwidth]{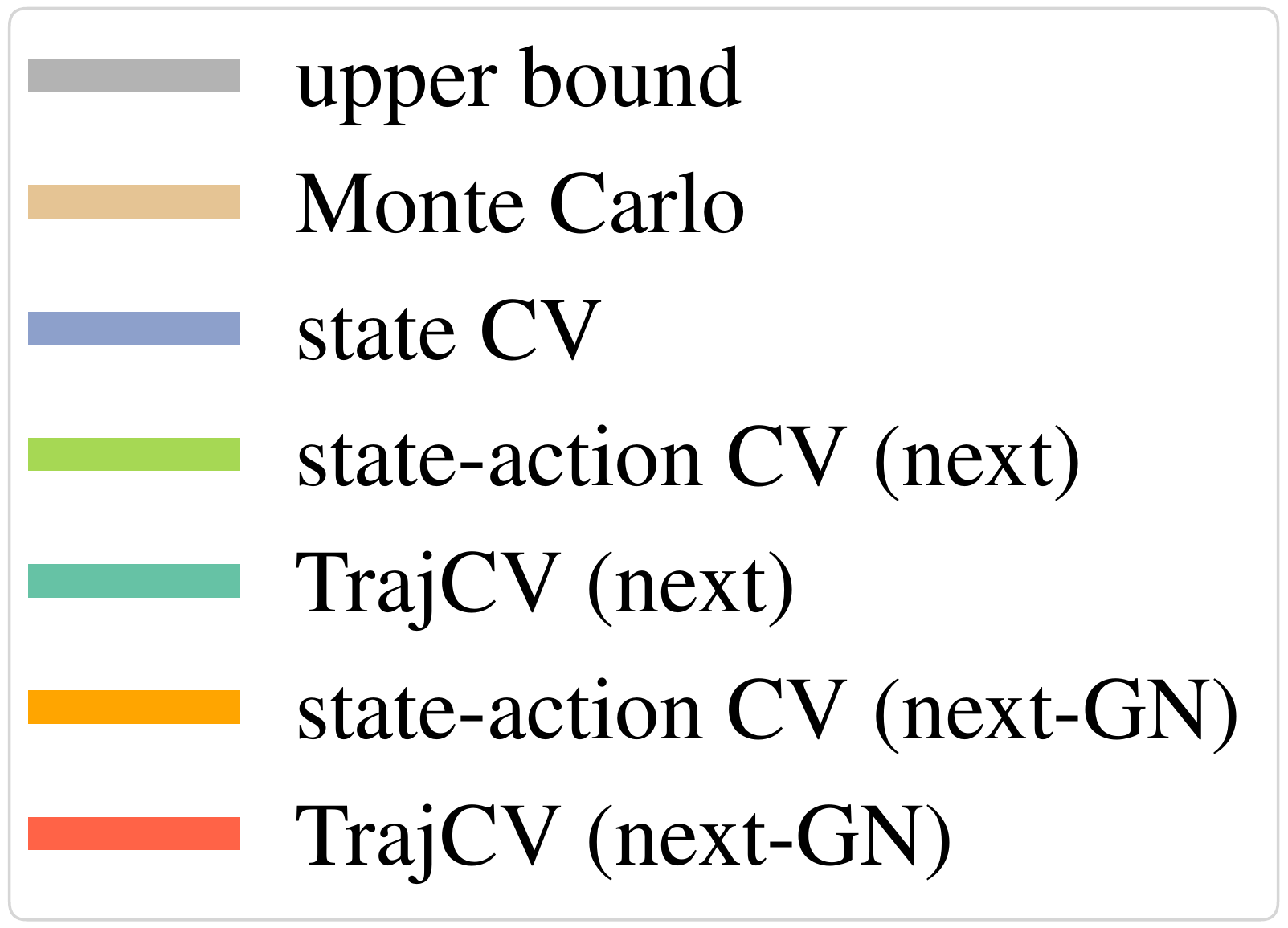}
		%\caption{CartPole4k}
		\label{fig:cp4k-next-legend}
	\end{subfigure}
	\caption{
		The exact same settings as \cref{fig:cps} except that 
		\emph{the state-action CV and TrajCV are given by $\wh{q}^{(\text{diff})}$ and $\wh{q}^{(\text{diff-GN})}$  (\cref{fig:cp4k-diff}), and  $\wh{q}^{(\text{next})}$ and $\wh{q}^{(\text{next-GN})}$ (\cref{fig:cp4k-next})}.
	}
	\label{fig:cp4k-more}
	%\vspace{-5mm}
\end{figure}

We compare naive Monte Carlo estimate \eqref{eq:policy gradient}, state dependent CV \eqref{eq:state cv}, state-action CV \eqref{eq:state-action cv}, and \mbox{TrajCV} \eqref{eq:traj-wise difference estimator}, in \cref{fig:cps} and \cref{fig:cp4k-more}.
To facilitate a fair comparison, we realize all algorithms with the same implementation
%the state CV is TrajCV with $\wh{q} = \wh{v}$, and state-action is TrajCV but with $\Delta_{t:h}^{(t)}$ replaced with $\Delta_{t}^{(t)}$. 
based on an on-policy value function approximator $\wh{v}$, so that the effects of value estimates' quality can be normalized. %; they all face the same type and degree of errors.
More concretely, $\wh{v}$ is learned by sampling abundant data from a biased dynamics model (which is obtained by perturbing the underlying physical parameters).
%where $\wh{\Exp}_{S_{t+1}|S_t,A_t}$ is similarly approximated by sampling the biased dynamics model multiple times

To construct the Q-function approximator $\wh q$ required by the state-action CV and TrajCV, 
we further train a deterministic function $\wh d$ that maps the current state and action to next state using the same data from the environment that are used for computing the policy gradient estimates. 
Combining the value function approximator $\wh v$ and the dynamics approximator $\wh d$ yields a natural Q-function approximator
$\wh{q}^{\text{(dyn)}}(s, a) = c(s, a) + \wh{v}(\wh{d}(s,a))$. 
Based on this basic $\wh{q}^{\text{(dyn)}}$, we explored several options 
of Q-function approximator for defining the state-action CV and TrajCV:
\begin{enumerate}
	\item Monte Carlo (MC) :  $\wh{q}^{\text{(dyn)}}(s,a)$. We use many samples of actions (1000 in the experiments) to approximate $\Exp_{A_t|S_t}  \br{\wh q^{\text{(dyn)}}(S_t, A_t)}$.
	To reduce variance,
	we use the same action randomness for different steps, \ie using the same 1000  i.i.d. samples from $p_R$ (defined in \cref{sec:optimal ordering}) in the evaluation for  $\Exp_{A_t|S_t}$ with different $t$.
	\item
	We also consider various  Q-function approximators that are quadratic in action, so that $\Exp_{A_t|S_t}$ can be evaluated in closed-form. They are 	derived by different linearizations of the Q-function approximator $\wh q$ as shown below.
%	Quadratic Q-function estimators for quadratic $c$, based on different linearizations %(\eg, a neural net that is quadratic in action) 
%	In (c), perhaps missing $c(s,a)$ and in the parenthesis after first $\nabla_m$, is $\wh v(\wh s')$ considered as a function of $m$?}
	\begin{enumerate}
		\item 
		$\wh{q}^{(\text{next})}(s, a) = c(s, a) + \wh{v}(\wh{s}') + (a-m)^\t \nabla_m \wh{d}(s,m) \nabla \wh{v}(\wh{s}'), $
		\item $\wh{q}^{(\text{next-GN})}(s, a) = \wh{q}^{(\text{next})}(s, a) + \frac{1}{2} (a-m)^\t \nabla_m \wh{d}(s,m) \nabla^2 \wh{v}(\wh{s}') \nabla_m \wh{d}(s,m)^\t (a-m), $
		\item $\wh{q}^{(\text{diff})}(s, a) = \wh{v}(s)  + (a-m)^\t \nabla_m (c(s, m) + \wh{v}(\wh{s}')) + \frac{1}{2}(a-m)^\t \nabla_m^2 c(s,m) (a-m),$
		\item $\wh{q}^{(\text{diff-GN})}(s, a) =\wh{q}^{(\text{diff})}(s, a) + \frac{1}{2} (a-m)^\t \nabla_m \wh{d}(s,m) \nabla^2 \wh{v}(\wh{s}') \nabla_m \wh{d}(s,m)^\t (a-m), $ 
	\end{enumerate}
	where $m = \mu_\theta(s)$ is the mean of the Gaussian policy, $\wh{s}' = \wh{d}(s,m)$,  and ``GN'' stands for Gauss-Newton. 
	We assume $c(s,a)$ is quadratic in $a$ for $\wh{q}^{(\text{next})}$ and $\wh{q}^{(\text{next-GN})}$.
\end{enumerate}
The performance of different CVs using 
MC for approximating $\Exp_{A_t|S_t}$ is reported in \cref{fig:cps}, where 5 rollouts are sampled for each iteration. We also emulate  noise-free gradients, denoted as upper bound, which is constructed by running the state dependent CV but with $100,000$ samples per iteration.\footnote{For the usual learners, the number of samples collected per iteration is often less than $5h$, and much less at the start of learning, because of early termination when the agent fails.}
Overall when more information is used to design the CVs (from state only, state-action, and then trajectory-wise) the convergence speed improves. In particular, as the problem horizon becomes longer, the gap becomes larger: the reward feedback becomes sparser, so the variance due to long-term trajectory starts to dominate, as shown in \cref{fig:cp4k}. 
In \cref{fig:cp4k-more}, 
we observe that when $\Exp_{A_t|S_t}$ is approximated analytically using various techniques ($\wh{q}^{(\text{next})}$ and  $\wh{q}^{(\text{next-GN})}$ are the CVs suggested in~\cite{gu2016q} and \cite{pankov2018reward}, respectively), learning is accelerated when more information is leveraged in CV synthesis.

These preliminary experimental results support the theoretical insights provided in \cref{sec:why we need new cvs} and \cref{sec:trajectory}, suggesting the importance of considering long-term effects in designing CVs, especially for problems with a long horizon. The fix turns out to be quite simple: just padding additional terms (cf. \eqref{eq:traj-wise difference estimator}) onto the existing CVs, which can be done using Q-function approximators available in existing CVs without  new information. Interestingly we prove this simple idea is optimal.
 Important future work includes considering the different bias and variance trade-off discussed in \cref{sec:optimal ordering}, and learning the linear combination weights of the CVs for policy gradient components $\{G_t\}_{t=1}^h$.

%
%\subsection{Simple Task}
%We can use a simple task that's quick to simulate, i.e. CartPole, to investigate many interesting properties of our CV.
%\begin{enumerate}
%	\item Following the work of~\citet{tucker2018mirage}, we can show that $\Sigma_\tau$ is indeed decreased with our CV. 
%	Choose several $t$, e.g. $\{0, 100, 200, 300\}$ at a certain iteration and compute variance components with different CVs. And use bar chart to show their size.
%	\item The impact of the accuracy of value function and action expectation evaluation. 
%	Use a different number of samples from the true environment to train the value function. 
%	And use a different number of actions to expectation computation. 
%	\item Different problem complexity.
%	Vary the horizon of CartPole, we can demonstrate our CV is more beneficial on the problems that has long-term uncertainty in actions.	
%\end{enumerate}
%
%
%
%\subsection{More Complex Tasks: Locomotion}
%On the more complex tasks,  we use more realistic setups. 

%\begin{enumerate}
%	\item Analytic form for expectation computation with linearized Q-function.
%\end{enumerate}

%===============================================================================

%\section{Conclusion}
%\label{sec:conclusion}

%===============================================================================

% The maximum paper length is 8 pages excluding references and acknowledgements, and 10 pages including references and acknowledgements

%\clearpage
% The acknowledgments are automatically included only in the final version of the paper.

\acknowledgments{This work was partially supported by NSF CAREER award 1750483.
}

%===============================================================================

\clearpage
\bibliography{references}  % .bib

\clearpage
\appendix
%\allowdisplaybreaks
\section*{Appendix}
\section{Missing Proofs}\label{app:proof}
\subsection{Proof for Theorem~\ref{thm: optimal ordering}} 

To understand how the ordering matters, we consider a toy example of estimating $\Exp_{X,Y}[f(X,Y)]$ of some function $f$ of two random variables $X$ and $Y$. We prove a basic lemma.
\begin{lemma} \label{lemma: two var ordering}
	If $X$ and $Y$ are independent, then
	\begin{align}
	\Var_X \Exp_{Y} \br{ f(X,Y)} \le \Exp_Y \Var_{X} \br{f(X,Y)}
	\end{align}
\end{lemma}
\begin{proof}
	This can be proved by Jensen's inequality.
	\begin{align*}
	\Var_X \Exp_{Y} \br{ f}
	= \Exp_X \paren{\Exp_Y \br{f - \Exp_X \br{f}}}^2
	\le \Exp_X \Exp_Y \br{\paren{ f - \Exp_{X,Y} \br{f}}^2}
	= \Exp_Y  \Var_X \br{f(X,Y)}
	\end{align*}	
\end{proof}
Suppose we want to reduce the variance of estimating $\Exp_{X,Y}[f(X,Y)]$ with some CV $g(X,Y)$ but only knowing the distribution $P(X)$, not $P(Y)$.
Lemma \ref{lemma: two var ordering} tells us that in decomposing the total variance of $f(X,Y)$ to design this CV (cf. \cref{sec:simple CV}) we should take the decomposition
\begin{align}
\Var_Y \Exp_{X} \br{f(X,Y)} + \Exp_Y \Var_{X} \br{f(X,Y)}
\end{align}
instead of the decomposition 
\begin{align}
\Var_X \Exp_{Y} \br{ f(X,Y)} + \Exp_X \Var_{Y} \br{ f(X,Y)} 
\end{align}
In other words, we should take the ordering $Y \to X$, instead of $X \to Y$, when we invoke the law of total variance. The reason is that after choosing the optimal CV for each case to reduce the variance due to $X$ (the information that we have access to), we are left with
$
\Var_Y \Exp_{X} \br{f(X,Y)}
$
and 
$
\Exp_X \Var_{Y} \br{ f(X,Y)} 
$, 
respectively, for  $Y \to X$ and $X \to Y$. By Lemma \ref{lemma: two var ordering}, we see the $Y \to X$ has a smaller residue in variance. 
In other words, when we only have partial information about the distribution, we should arrange the random variables whose distribution we know to the latter stage of the ordering, so that the CV we design can leverage the sampled observations to compensate for the lack of prior.

%Therefore, although $\Var_X \Exp_{Y} \br{ f}$ and 
%$\Exp_Y \Var_{X} \br{f}$ look similar %both are associated with $X$, 
%ordering variance \emph{later} presents a larger potential for variance reduction.  

We use this idea to prove the natural ordering \cref{eq:natural ordering} in optimal. In analogy of $X$ and $Y$, we have the action randomness whose distribution is known (\ie the policy) and the dynamics randomness, whose distribution is unknown.

The potential orderings we consider come from first reparameterizing the policy and then ordering the independent random variables $R_t$ (cf. \cref{sec:optimal ordering}). 
We note that the CV is determined by the  ordering, not due to reparameterization. For the natural ordering,
\begin{align} \tag{\ref{eq:natural ordering}}
S_t \rightarrow A_t \rightarrow S_{t+1} \rightarrow A_{t+1} \rightarrow \dots \rightarrow S_{h} \rightarrow A_{h},
\end{align}
it gives the same control variate of the ordering below based on reparameterization
\begin{align} \label{eq:natural ordering in R}
S_t \rightarrow R_t \rightarrow S_{t+1} \rightarrow R_{t+1} \rightarrow \dots \rightarrow S_{h} \rightarrow R_{h}.
\end{align}

Suppose that given an ordering, we can compute its optimal CV. We define the variance left after applying that optimal CV associated with the ordering, the \emph{residue} of that ordering. 
We will show that the residue is minimized at the natural ordering.

The proof consists of two steps.
\begin{enumerate}
\item 
We show that when dynamics is the MDP is unknown, 
an ordering is \emph{feasible}, if and only if, $R_k$ appears before $S_{k+1, h}$ for all $t \le k< h$. That is, a feasible ordering must be causal at least in actions: the action randomness that causes a state must be arranged before that state in the ordering.
We prove this  by contradiction. 
Assume otherwise $S_u$ is the \emph{first} state before $R_k$ satisfying $u > k$.
We see that $R_k$ and $S_{u}$ are dependent, if none of the variables in $S_{k+1,u}$ is given.
%even when all other variables excluding $S_{k+1,u}$ are given.
This observation can been inferred from 
the Bayes network that connect these random variables (\cref{fig: bayes net 2}), \ie the path from $R_t$ to $S_u$ is not blocked unless any of $S_{k+1,u}$ is observed~\cite{bishop2006pattern}. 
Therefore, if we have an ordering that is violates the causality property defined above, the expectation over $R_k$ required to define the difference estimator becomes intractable to compute, because the dynamics is \emph{unknown}.
This creates a contradiction.

\item 
We show that any feasible ordering can be transformed into the natural ordering in \eqref{eq:natural ordering} using operations that do not increase the residue. 
We consider the following operations
\begin{enumerate}
\item Suppose, in an ordering, there is $S_v \rightarrow S_u, v > u$, then we can exchange them without affecting residue.  

\item Suppose, in a feasible ordering, there is $S_v \rightarrow R_k \rightarrow S_u$ with $v > u$  and $k\neq u,v$. 
Because this is a feasible ordering, we have $k+1\leq u < v$. %$u < v \le k$.
This means that we can also move $R_k$ after $S_u$. This change would not increase residue, because of the discussion after Lemma~\ref{lemma: two var ordering}.  Then we change exchange the order of $S_v$ and $S_u$ too using the first operation.
\end{enumerate}
By using these two operations repeatedly, we can make all the states ordered by their subscripts, without increasing the residue. Finally, we can move $R_k$ to just right after $S_k$ without increasing residue using Lemma~\ref{lemma: two var ordering} again. 
Thus, we arrive at the natural ordering in \eqref{eq:natural ordering in R}, which is the same as \eqref{eq:natural ordering}. This concludes the proof.

\end{enumerate}
\subsection{Proof of \cref{th:varaince size}}

Suppose the dimension of $\AA$ is $d_{\AA}$ which is finite.
To bound these terms, we derive some intermediate bounds.
First, by the Gaussian assumption, 
\begin{align*}
\pi_{S_t}(A_t) = (2\pi\sigma)^{-\frac{d_{\AA}}{2}} \exp\left( \frac{-1
}{2\sigma} \norm{A_t - \mu_\theta(S_t)}^2 \right)
\end{align*}
we see that 
\begin{align*}
N_t := \nabla \ln \pi(A_t|S_t) 
= 
	\begin{bmatrix}
	\nabla_\theta  \ln \pi(A_t|S_t) \\
	\nabla_\sigma  \ln \pi(A_t|S_t)
	\end{bmatrix}
= \begin{bmatrix}
\frac{-1}{\sigma} \nabla \mu_\theta(S_t)(A_t - \mu_\theta(S_t)) \\
\frac{1
}{2\sigma^2} \norm{A_t - \mu_\theta(S_t)}^2 -\frac{d_{\AA}}{2\sigma}
\end{bmatrix}
\end{align*}
Therefore, for $\sigma$ small enough, $\norm{N_t} = O(\frac{\poly(A_t)}{\sigma^2})$. 

Second, by the assumption on boundedness of $C$, we have $C_{t:h} = O(h)$ and $Q_t \coloneqq q^{\pi}(S_t,A_t) =  O(h)$.
We use these equalities to bound $\Exp_{|S_t} \br{ N_t C_{t:h}}$. 
We observe that the identity that \begin{align*}
\Exp_{|S_t} \br{ N_t C_{t:h}} =  \nabla \Exp_{A_t|S_t}[q^{\pi}(S_t, A_t)] 
\end{align*}
Under the assumption that $q^{\pi}$ is analytic, $q^{\pi}$ can be written in terms of an infinite sum of polynomials, \ie $q^\pi(S_t, A_t) = \poly_{S_t}(A_t)$, where the subscript denotes the coefficients in the polynomial depends on $S_t$.

Now we are ready to bound $\V_{S_t}$, 
$\V_{A_t|S_t}$, and $
\V_{|S_t, A_t}$.  We recall that the expectation of polynomials over a Gaussian distribution depends only polynomially on the Gaussian's variance, with an order no less than $1$.  
Therefore, for $\sigma$ small enough, we have
$\norm{ \nabla \Exp_{A_t|S_t}[q(S_t, A_t)] } = O(h)$ independent of $\sigma$, which implies that 
\begin{align*}
\V_{S_t} &= \Tr\paren{\Var_{S_t} \br{ \Exp_{|S_t} \br{ N_t C_{t:h}}}} = o(h^2)
\end{align*}
We can apply the same observation on the Gaussian expectation of polynomials and derive, for $\sigma$ small enough,
\begin{align*}
\V_{A_t|S_t}
&= \Tr\paren{	\Exp_{S_t} \br{ \Var_{A_t|S_t} \br{  N_t \paren{\Exp_{|S_t, A_t} \br { C_{t:h}} }}}} \\
&= \Tr\paren{	\Exp_{S_t} \br{ \Var_{A_t|S_t} \br{  N_t q^{\pi}(S_t, A_t)}}} = O\left(\frac{h^2}{\sigma^4} \right)
\end{align*}
Similarly we can show
\begin{align*}
\V_{|S_t, A_t} &= \Tr\paren{\Exp_{S_t, A_t} [\Var_{|S_t, A_t} \br {N_t C_{t:h} }]} = O\left(\frac{h^2}{\sigma^4} \right)
\end{align*}
This concludes the proof.

\subsection{Bound for Variance of Policy Gradient}
\label{app:pg var bound}
The variance of the policy gradient $\Var[G]$ can be bounded by the variance of policy gradient components $\{\Var[G_t]\}_{t=1}^n$.
Appealing to the formula for the variance of the sum of two random variables
\begin{align*}
\Var[X+Y] = \Var[X] + \Var[Y] + 2 \Cov[X, Y],
\end{align*}
linearity of covariance
\begin{align*}
\Cov[X, Y + Z] = \Cov[X, Y] + \Cov[X, Z]
\end{align*}
 and Cauchy -Schwartz inequality
 \begin{align*}
\Cov[X,Y] \le \Var[X] + \Var[Y],
 \end{align*}
we can derive the following:
\begin{align*}
\eqsp \Var[G] 
&= \Var \br{ G_{1:h}}
\\&= \Var \br{G_1} + \Var \br{G_{2:h}} + \Cov[G_1, G_{2:h}]
\\&=\Var [G_1] + \sum_{t=2}^h \Cov[G_1, G_t] + \Var \br{G_{2:h}}
\\&= \sum_{t=1}^h \Var[G_t] + \sum_{u = 1}^h \sum_{v = u+1}^h \Cov[G_u, G_v]
\\& \le \sum_{t=1}^h \Var[G_t]  + \sum_{u = 1}^h \sum_{v = u+1}^h \paren{\Var[G_u] + \Var[G_v]}
\\& = h \sum_{t=1}^h \Var[G_t]
\end{align*}

\end{document}

%% file: shortcuts.tex
\usepackage{amsmath}
\usepackage{amsfonts}
\usepackage{amssymb}
\usepackage{amsthm}
\usepackage{bm}
\usepackage{bbm}
\usepackage{mathtools}
\usepackage{enumitem}
\usepackage{thmtools,thm-restate}
\usepackage{algorithm}
\usepackage{algorithmic}
\usepackage[capitalise]{cleveref}
\usepackage{color}
\usepackage[dvipsnames]{xcolor}
\usepackage{graphicx}
\usepackage{comment}
\def\AA{\mathcal{A}}

\def\NN{\mathcal{N}}

\def\SS{\mathcal{S}}

\def\Rbb{\mathbb{R}}

\def\tv{\boldsymbol{t}}

\def\R{\Rbb}

\def\t{\top}
\def\*{\star}

\renewcommand{\eg}{e.g.\xspace}
\renewcommand{\ie}{i.e.\xspace}

%% file: newcv.bbl
\begin{thebibliography}{44}
\providecommand{\natexlab}[1]{#1}
\providecommand{\url}[1]{\texttt{#1}}
\expandafter\ifx\csname urlstyle\endcsname\relax
  \providecommand{\doi}[1]{doi: #1}\else
  \providecommand{\doi}{doi: \begingroup \urlstyle{rm}\Url}\fi

\bibitem[Williams(1992)]{williams1992simple}
R.~J. Williams.
\newblock Simple statistical gradient-following algorithms for connectionist
  reinforcement learning.
\newblock \emph{Machine learning}, 8\penalty0 (3-4):\penalty0 229--256, 1992.

\bibitem[Sutton et~al.(2000)Sutton, McAllester, Singh, and
  Mansour]{sutton2000policy}
R.~S. Sutton, D.~A. McAllester, S.~P. Singh, and Y.~Mansour.
\newblock Policy gradient methods for reinforcement learning with function
  approximation.
\newblock In \emph{Advances in Neural Information Processing Systems}, pages
  1057--1063, 2000.

\bibitem[Kakade(2002)]{kakade2002natural}
S.~M. Kakade.
\newblock A natural policy gradient.
\newblock In \emph{Advances in Neural Information Processing Systems}, pages
  1531--1538, 2002.

\bibitem[Peters and Schaal(2008)]{peters2008natural}
J.~Peters and S.~Schaal.
\newblock Natural actor-critic.
\newblock \emph{Neurocomputing}, 71\penalty0 (7-9):\penalty0 1180--1190, 2008.

\bibitem[Schulman et~al.(2015)Schulman, Levine, Abbeel, Jordan, and
  Moritz]{schulman2015trust}
J.~Schulman, S.~Levine, P.~Abbeel, M.~Jordan, and P.~Moritz.
\newblock Trust region policy optimization.
\newblock In \emph{International Conference on Machine Learning}, pages
  1889--1897, 2015.

\bibitem[Cheng et~al.(2019)Cheng, Yan, Ratliff, and Boots]{cheng2018predictor}
C.-A. Cheng, X.~Yan, N.~Ratliff, and B.~Boots.
\newblock Predictor-corrector policy optimization.
\newblock In \emph{International Conference on Machine Learning}, 2019.

\bibitem[Konda and Tsitsiklis(2000)]{konda2000actor}
V.~R. Konda and J.~N. Tsitsiklis.
\newblock Actor-critic algorithms.
\newblock In \emph{Advances in Neural Information Processing Systems}, pages
  1008--1014, 2000.

\bibitem[Cheng et~al.(2018)Cheng, Yan, Wagener, and Boots]{cheng2018fast}
C.-A. Cheng, X.~Yan, N.~Wagener, and B.~Boots.
\newblock Fast policy learning through imitation and reinforcement.
\newblock In \emph{Conference on Uncertainty in Artificial Intelligence}, 2018.

\bibitem[Yang and Zhang(2019)]{yang2019policy}
L.~Yang and Y.~Zhang.
\newblock Policy optimization with stochastic mirror descent.
\newblock \emph{arXiv preprint arXiv:1906.10462}, 2019.

\bibitem[Ghadimi et~al.(2016)Ghadimi, Lan, and Zhang]{ghadimi2016mini}
S.~Ghadimi, G.~Lan, and H.~Zhang.
\newblock Mini-batch stochastic approximation methods for nonconvex stochastic
  composite optimization.
\newblock \emph{Mathematical Programming}, 155\penalty0 (1-2):\penalty0
  267--305, 2016.

\bibitem[Kimura et~al.(2000)Kimura, Kobayashi, et~al.]{kimura2000analysis}
H.~Kimura, S.~Kobayashi, et~al.
\newblock An analysis of actor-critic algorithms using eligibility traces:
  reinforcement learning with imperfect value functions.
\newblock \emph{Journal of Japanese Society for Artificial Intelligence},
  15\penalty0 (2):\penalty0 267--275, 2000.

\bibitem[Thomas(2014)]{thomas2014bias}
P.~Thomas.
\newblock Bias in natural actor-critic algorithms.
\newblock In \emph{International Conference on Machine Learning}, pages
  441--448, 2014.

\bibitem[Silver et~al.(2014)Silver, Lever, Heess, Degris, Wierstra, and
  Riedmiller]{silver2014deterministic}
D.~Silver, G.~Lever, N.~Heess, T.~Degris, D.~Wierstra, and M.~Riedmiller.
\newblock Deterministic policy gradient algorithms.
\newblock In \emph{International Conference on Machine Learning}, 2014.

\bibitem[Schulman et~al.(2016)Schulman, Moritz, Levine, Jordan, and
  Abbeel]{schulman2015high}
J.~Schulman, P.~Moritz, S.~Levine, M.~Jordan, and P.~Abbeel.
\newblock High-dimensional continuous control using generalized advantage
  estimation.
\newblock In \emph{International Conference on Learning Representations}, 2016.

\bibitem[Sun et~al.(2018)Sun, Bagnell, and Boots]{sun2018truncated}
W.~Sun, J.~A. Bagnell, and B.~Boots.
\newblock Truncated horizon policy search: Combining reinforcement learning \&
  imitation learning.
\newblock In \emph{International Conference on Learning Representations}, 2018.

\bibitem[Efroni et~al.(2019)Efroni, Dalal, Scherrer, and
  Mannor]{efroni2018beyond}
Y.~Efroni, G.~Dalal, B.~Scherrer, and S.~Mannor.
\newblock Beyond the one step greedy approach in reinforcement learning.
\newblock In \emph{International Conference on Machine Learning}, 2019.

\bibitem[Ng et~al.(1999)Ng, Harada, and Russell]{ng1999policy}
A.~Y. Ng, D.~Harada, and S.~Russell.
\newblock Policy invariance under reward transformations: Theory and
  application to reward shaping.
\newblock In \emph{International Conference on Machine Learning}, volume~99,
  pages 278--287, 1999.

\bibitem[Greensmith et~al.(2004)Greensmith, Bartlett, and
  Baxter]{greensmith2004variance}
E.~Greensmith, P.~L. Bartlett, and J.~Baxter.
\newblock Variance reduction techniques for gradient estimates in reinforcement
  learning.
\newblock \emph{Journal of Machine Learning Research}, 5\penalty0
  (Nov):\penalty0 1471--1530, 2004.

\bibitem[Jie and Abbeel(2010)]{jie2010connection}
T.~Jie and P.~Abbeel.
\newblock On a connection between importance sampling and the likelihood ratio
  policy gradient.
\newblock In \emph{Advances in Neural Information Processing Systems}, pages
  1000--1008, 2010.

\bibitem[Gu et~al.(2017)Gu, Lillicrap, Ghahramani, Turner, and Levine]{gu2016q}
S.~Gu, T.~Lillicrap, Z.~Ghahramani, R.~E. Turner, and S.~Levine.
\newblock Q-prop: Sample-efficient policy gradient with an off-policy critic.
\newblock In \emph{International Conference on Learning Representations}, 2017.

\bibitem[Liu et~al.(2018)Liu, Feng, Mao, Zhou, Peng, and Liu]{liu2017action}
H.~Liu, Y.~Feng, Y.~Mao, D.~Zhou, J.~Peng, and Q.~Liu.
\newblock Action-depedent control variates for policy optimization via stein's
  identity.
\newblock In \emph{International Conference on Learning Representations}, 2018.

\bibitem[Grathwohl et~al.(2018)Grathwohl, Choi, Wu, Roeder, and
  Duvenaud]{grathwohl2017backpropagation}
W.~Grathwohl, D.~Choi, Y.~Wu, G.~Roeder, and D.~Duvenaud.
\newblock Backpropagation through the void: Optimizing control variates for
  black-box gradient estimation.
\newblock In \emph{International Conference on Learning Representations}, 2018.

\bibitem[Tucker et~al.(2018)Tucker, Bhupatiraju, Gu, Turner, Ghahramani, and
  Levine]{tucker2018mirage}
G.~Tucker, S.~Bhupatiraju, S.~Gu, R.~E. Turner, Z.~Ghahramani, and S.~Levine.
\newblock The mirage of action-dependent baselines in reinforcement learning.
\newblock \emph{arXiv preprint arXiv:1802.10031}, 2018.

\bibitem[Pankov(2018)]{pankov2018reward}
S.~Pankov.
\newblock Reward-estimation variance elimination in sequential decision
  processes.
\newblock \emph{arXiv preprint arXiv:1811.06225}, 2018.

\bibitem[Wu et~al.(2018)Wu, Rajeswaran, Duan, Kumar, Bayen, Kakade, Mordatch,
  and Abbeel]{wu2018variance}
C.~Wu, A.~Rajeswaran, Y.~Duan, V.~Kumar, A.~M. Bayen, S.~Kakade, I.~Mordatch,
  and P.~Abbeel.
\newblock Variance reduction for policy gradient with action-dependent
  factorized baselines.
\newblock In \emph{International Conference on Learning Representation}, 2018.

\bibitem[Bellman(1957)]{bellman1957markovian}
R.~Bellman.
\newblock A {M}arkovian decision process.
\newblock \emph{Journal of Mathematics and Mechanics}, pages 679--684, 1957.

\bibitem[Bertsekas et~al.(1995)Bertsekas, Bertsekas, Bertsekas, and
  Bertsekas]{bertsekas1995dynamic}
D.~P. Bertsekas, D.~P. Bertsekas, D.~P. Bertsekas, and D.~P. Bertsekas.
\newblock \emph{Dynamic programming and optimal control}, volume~1.
\newblock Athena scientific Belmont, MA, 1995.

\bibitem[Beck and Teboulle(2003)]{beck2003mirror}
A.~Beck and M.~Teboulle.
\newblock Mirror descent and nonlinear projected subgradient methods for convex
  optimization.
\newblock \emph{Operations Research Letters}, 31\penalty0 (3):\penalty0
  167--175, 2003.

\bibitem[Kakade et~al.(2003)]{kakade2003sample}
S.~M. Kakade et~al.
\newblock \emph{On the sample complexity of reinforcement learning}.
\newblock PhD thesis, University of London London, England, 2003.

\bibitem[Vemula et~al.(2019)Vemula, Sun, and Bagnell]{vemula2019contrasting}
A.~Vemula, W.~Sun, and J.~A. Bagnell.
\newblock Contrasting exploration in parameter and action space: A zeroth-order
  optimization perspective.
\newblock In \emph{International Conference on Artificial Intelligence and
  Statistics}, 2019.

\bibitem[Ross(1990)]{ross1990course}
S.~M. Ross.
\newblock \emph{A course in simulation}.
\newblock Prentice Hall PTR, 1990.

\bibitem[Owen(2013)]{mcbook}
A.~B. Owen.
\newblock \emph{Monte Carlo theory, methods and examples}.
\newblock 2013.

\bibitem[Schmidt et~al.(2017)Schmidt, Le~Roux, and Bach]{schmidt2017minimizing}
M.~Schmidt, N.~Le~Roux, and F.~Bach.
\newblock Minimizing finite sums with the stochastic average gradient.
\newblock \emph{Mathematical Programming}, 162\penalty0 (1-2):\penalty0
  83--112, 2017.

\bibitem[Johnson and Zhang(2013)]{johnson2013accelerating}
R.~Johnson and T.~Zhang.
\newblock Accelerating stochastic gradient descent using predictive variance
  reduction.
\newblock In \emph{Advances in Neural Information Processing Systems}, pages
  315--323, 2013.

\bibitem[Defazio et~al.(2014)Defazio, Bach, and
  Lacoste-Julien]{defazio2014saga}
A.~Defazio, F.~Bach, and S.~Lacoste-Julien.
\newblock Saga: A fast incremental gradient method with support for
  non-strongly convex composite objectives.
\newblock In \emph{Advances in Neural Information Processing Systems}, pages
  1646--1654, 2014.

\bibitem[Wang et~al.(2013)Wang, Chen, Smola, and Xing]{wang2013variance}
C.~Wang, X.~Chen, A.~J. Smola, and E.~P. Xing.
\newblock Variance reduction for stochastic gradient optimization.
\newblock In \emph{Advances in Neural Information Processing Systems}, pages
  181--189, 2013.

\bibitem[Ciosek and Whiteson(2018)]{ciosek2018expected}
K.~Ciosek and S.~Whiteson.
\newblock Expected policy gradients for reinforcement learning.
\newblock \emph{arXiv preprint arXiv:1801.03326}, 2018.

\bibitem[Singh and Sutton(1996)]{singh1996reinforcement}
S.~P. Singh and R.~S. Sutton.
\newblock Reinforcement learning with replacing eligibility traces.
\newblock \emph{Machine learning}, 22\penalty0 (1-3):\penalty0 123--158, 1996.

\bibitem[Chung(2001)]{chung2001course}
K.~L. Chung.
\newblock \emph{A course in probability theory}.
\newblock Academic press, 2001.

\bibitem[Baxter and Bartlett(2001)]{baxter2001infinite}
J.~Baxter and P.~L. Bartlett.
\newblock Infinite-horizon policy-gradient estimation.
\newblock \emph{Journal of Artificial Intelligence Research}, 15:\penalty0
  319--350, 2001.

\bibitem[Landau and Lifshitz(1958)]{landau1958statistical}
L.~Landau and E.~Lifshitz.
\newblock Statistical physics (course of theoretical physics vol 5).
\newblock 1958.

\bibitem[Brockman et~al.(2016)Brockman, Cheung, Pettersson, Schneider,
  Schulman, Tang, and Zaremba]{brockman2016openai}
G.~Brockman, V.~Cheung, L.~Pettersson, J.~Schneider, J.~Schulman, J.~Tang, and
  W.~Zaremba.
\newblock Open{AI} {G}ym.
\newblock \emph{arXiv preprint arXiv:1606.01540}, 2016.

\bibitem[Lee et~al.(2018)Lee, Grey, Ha, Kunz, Jain, Ye, Srinivasa, Stilman, and
  Liu]{Lee2018}
J.~Lee, M.~X. Grey, S.~Ha, T.~Kunz, S.~Jain, Y.~Ye, S.~S. Srinivasa,
  M.~Stilman, and C.~K. Liu.
\newblock {DART}: Dynamic animation and robotics toolkit.
\newblock \emph{The Journal of Open Source Software}, 3\penalty0 (22):\penalty0
  500, feb 2018.

\bibitem[Bishop(2006)]{bishop2006pattern}
C.~M. Bishop.
\newblock \emph{Pattern recognition and machine learning}.
\newblock springer, 2006.

\end{thebibliography}
